\documentclass{article}
\usepackage[preprint]{log_2022}

\usepackage{booktabs}           
\usepackage{multirow}           
\usepackage{amsfonts}           
\usepackage{graphicx}           
\usepackage{duckuments}         

\usepackage[numbers,compress,sort]{natbib}

\usepackage{graphicx}
\usepackage{amsmath,amsthm,amssymb,amsfonts}
\usepackage[italicdiff]{physics}
\usepackage[T1]{fontenc}
\usepackage{wrapfig}
\usepackage{lmodern,mathrsfs}
\usepackage[inline,shortlabels]{enumitem}

\usepackage{algorithm}
\usepackage{algpseudocode}

\algdef{SE}[SUBALG]{Indent}{EndIndent}{}{\algorithmicend\ }%
\algtext*{Indent}
\algtext*{EndIndent}

\usepackage{mathrsfs}
\usepackage{amssymb}
\usepackage{amsfonts}
\usepackage{amsmath}
\usepackage{amsthm,verbatim}

\newcommand{\bA}{{\boldsymbol A}}
\newcommand{\bB}{{\boldsymbol B}}
\newcommand{\bC}{{\boldsymbol C}}
\newcommand{\bD}{{\boldsymbol D}}

\newcommand{\bF}{{\boldsymbol F}}

\newcommand{\bH}{{\boldsymbol H}}
\newcommand{\bL}{{\boldsymbol L}}

\newcommand{\bI}{{\boldsymbol I}}

\newcommand{\bR}{{\boldsymbol R}}

\newcommand{\bW}{{\boldsymbol W}}
\newcommand{\bX}{{\boldsymbol X}}
\newcommand{\bY}{{\boldsymbol Y}}
\newcommand{\bZ}{{\boldsymbol Z}}

\newcommand{\bHatA}{{\boldsymbol{\Hat{A}} }}
\newcommand{\bTildeA}{{\boldsymbol{\Tilde{A}} }}

\newcommand{\bXi}{{\boldsymbol \Xi}}

\newcommand{\bOnes}{{\boldsymbol 1}}
\newcommand{\bZeros}{{\boldsymbol 0}}

\newcommand{\ba}{{\boldsymbol a}}

\newcommand{\bg}{{\boldsymbol g}}
\newcommand{\bh}{{\boldsymbol h}}

\newcommand{\bu}{{\boldsymbol u}}
\newcommand{\bv}{{\boldsymbol v}}

\newcommand{\bx}{{\boldsymbol x}}
\newcommand{\by}{{\boldsymbol y}}
\newcommand{\bz}{{\boldsymbol z}}

\newcommand{\BR}{\mathbb{R}}
\newcommand{\BE}{\mathbb{E}}

\newcommand{\BP}{\mathbb{P}}

\newcommand{\mB}{{\cal B}}

\newcommand{\mD}{{\cal D}}
\newcommand{\mE}{{\cal E}}

\newcommand{\mG}{{\cal G}}

\newcommand{\mM}{{\cal M}}

\newcommand{\mO}{{\cal O}}

\newcommand{\mV}{{\cal V}}

\newcommand{\MLP}{{\mbox{MLP}}}
\newcommand{\Softmax}{{\mbox{Softmax}}}

\newcommand{\bepsilon}{{\boldsymbol \epsilon}}

\newcommand{\bxi}{{\boldsymbol \xi}}

\def\boxit#1{\vbox{\hrule\hbox{\vrule\kern6pt
          \vbox{\kern6pt#1\kern6pt}\kern6pt\vrule}\hrule}}

\newtheorem{assumption}{Assumption}
\newtheorem{theorem}{Theorem}[section]
\newtheorem{lemma}{Lemma}[section]

\newtheorem{corollary}[theorem]{Corollary}

\newtheorem{remark}{Remark}

\title[Graph Federated Learning with Hidden Representation Sharing]{Graph Federated Learning with Hidden Representation Sharing}

\author[S. Wu et al.]{%
Shuang Wu
\thanks{Equal Contribution.}
\thanks{Corresponding Author.}\\
\institute{UCLA}\\
\email{shuangwu222@ucla.edu}\And
Mingxuan Zhang\footnotemark[1]\\
\institute{Purdue University}\\
\email{zhan3692@purdue.edu}
\And
Yuantong Li\\
\institute{UCLA}\\
\email{yuantongli@ucla.edu}
\And
Carl Yang\\
\institute{Emory University}\\
\email{j.carlyang@emory.edu}
\And
Pan Li\\
\institute{Purdue University}\\
\email{panli@purdue.edu}
}

\begin{document}

\maketitle

\begin{abstract}

Learning on Graphs (LoG) is widely used in multi-client systems when each client has insufficient local data, and multiple clients have to share their raw data to learn a model of good quality. One scenario is to recommend items to clients with limited historical data and sharing similar preferences with other clients in a social network. On the other hand, due to the increasing demands for the protection of clients' data privacy, Federated Learning (FL) has been widely adopted: FL requires models to be trained in a multi-client system and restricts sharing of raw data among clients. The underlying potential data-sharing conflict between LoG and FL is under-explored and how to benefit from both sides is a promising problem. In this work, we first formulate the Graph Federated Learning (GFL) problem that unifies LoG and FL in multi-client systems and then propose sharing hidden representation instead of the raw data of neighbors to protect data privacy as a solution. To overcome the biased gradient problem in GFL, we provide a gradient estimation method and its convergence analysis under the non-convex objective. In experiments, we evaluate our method in classification tasks on graphs. Our experiment shows a good match between our theory and the practice.
\end{abstract}

\section{Introduction}\label{Introduction}

Learning on Graphs (LoG) in multi-client systems has extensive applications such as Graph Neural Networks (GNNs) for recommendation \citep{fan2019graph, ge2020graph, wang2020global, wu2021fedgnn}, finance \citep{wang2019semi, liu2018heterogeneous}, and traffic \citep{yu2017spatio, diehl2019graph}. 
The key to the success of LoG is sharing local raw data between clients. For example, when recommending items to users with insufficient local data, data sharing from their friends with similar preferences in a social network can improve the performance of recommendation models. 
On the other hand, Federated Learning (FL) is widely explored due to its protection of data privacy, especially in medical fields \citep{xu2020federated}, mobile device fields \citep{hard2018federated, lim2020federated}, and Internet of Things (IoT) fields \citep{lu2019blockchain}. In FL, models are trained without data sharing among clients to protect clients' local data privacy. As a consequence, combining FL and LoG in multi-client systems faces a fundamental conflict in data sharing.

As we know, considerable works are combing Federated Learning and Graph Machine Learning. One attractive research line is using FL to train GNNs \citep{wang2020graphfl,zhang2021subgraph}. In addition, \citep{meng2021cross, wu2021fedgnn} use FL to train GNN-based models to address specific real-world applications. \citep{zhang2021federated, he2021fedgraphnn} summarises current efforts on FL over graphs, including the above literature. However, most current works did not utilize the network of clients in the system and failed to protect the privacy of the nodes in the network. In other words, previous literature never models FL clients as nodes in GNNs on multi-client systems. Besides, all these works are application-oriented without a theoretical guarantee. Therefore, fundamental data sharing conflict remains unsolved.

Such significant conflict motivates our investigation of the construction of Graph Federated Learning (GFL) in multi-client systems: \textbf{Can we formulate a GFL framework to address the data sharing conflict, paired with theoretical and empirical supports?} We aim to deliver a generic framework of GFL. Our work focuses on the centralized federated learning setting while data collected by clients are Non-IID distributed. 

\textbf{Contributions.} 
We formulate the GFL problem for a graph-based model in multi-client systems. To address the data sharing conflict, we propose an FL solution with the hidden representation sharing technique, which only requires the sharing of hidden representations rather than the raw data from the neighbors to protect data privacy on multi-client systems. A technical challenge arises since the hidden representations are only exchanged during communication between clients and the central server, making unbiased gradient estimation becomes impractical. As a remedy, we provide a practical gradient estimation method. 
Moreover, a convergence analysis with non-convex objectives of the proposed algorithm is provided. 
To the best of our knowledge, this is the first theoretical analysis for FL with a graph-based model. 
We propose \texttt{GFL-APPNP} and empirically evaluate the proposed method for several classification tasks on graphs, including deterministic node classification, stochastic node classification, and supervised classification.
Our experiments show that the proposed method converges and achieves competitive performance. 
Additionally, the results provide a good match between our theory and practice. The contributions in this paper are summarized as follows:

\begin{itemize}[leftmargin=10pt, itemsep = -3pt]
    \item Formulate the GFL problem to model FL clients as nodes in LoG on multi-client systems.
    
    \item Propose FL solution with hidden representation sharing for GFL problem to resolve data sharing conflict. 
    
    \item Provide theoretical non-convex convergence analysis for GFL.
    
    \item Propose \texttt{GFL-APPNP} and empirically show the proposed algorithm is valid and has competitive performance on classification tasks.
    
\end{itemize}

\section{Related Works}\label{Related_works}
\textbf{Federated Learning for GNNs.} How to utilize the FL technique to train GNNs is an interesting topic that attracts lots of attention from researchers. For instance, \cite{wang2020graphfl} focuses on graph semi-supervised learning via meta-learning and handles testing nodes with new label domains as well as leverages unlabeled data. \citep{zhang2021subgraph} proposes federated learning to train GNNs by dividing a large graph into subgraphs. \citep{xie2021federated} considers an FL solution to train GNNs for the entire graph classification. \citep{he2021spreadgnn} proposes decentralized periodic SGD to solve the serverless Federated Multi-Task Learning for GNNs. \cite{meng2021cross} Proposes a GNN-based federated learning architecture for spatio-temporal data modeling. \citep{wu2021fedgnn} puts forward a decentralized federated framework for privacy-preserving GNN-based recommendations.
However, \citep{zhang2021subgraph, wang2020graphfl, xie2021federated, he2021fedgraphnn} assume each client has its own graphs and \citep{meng2021cross, wu2021fedgnn} use federated learning to train GNN-based model.  None of these works is federated learning to train GNNs on multi-client systems with the protection of node-level privacy, which is addressed by our work.

\textbf{Personalized Federated Learning.} The conventional FL approach faces a fundamental challenge of poor performance on highly heterogeneous clients. Previous works \citep{li2019convergence, li2020federated} provided solutions to tackling Non-IID data across clients. Recently, inspired by personalization, research on personalized federated learning has developed rapidly \citep{fallah2020personalized, t2020personalized, mansour2020three}. Particularly, personalization with graph structure to tackle heterogeneity in FL is highly related to our work. For example, \citep{smith2017federated} proposes \texttt{MOCHA} which uses a graph-type regularization to control the parameters and perform a prime-dual framework, and \citep{hanzely2020federated} provides a similar regularizer to the multitask learning. \citep{t2020personalized} considers an implicit model where personalized parameters come from a Moreau envelope, and this idea recently has got generalized to graph Laplacian regularization \citep{dinh2021new}. \citep{lalitha2019peer, lalitha2018fully} assumes that there is a common parameter shared across the network when each node of the graph is viewed as a federated learning client that generates independent data. All these works are model-level personalization based on graphs, such as graph regularization, while LoG encourages data-level sharing. 

\textbf{Notations.} 
Let $[n]$ be the set $\{1,2,...,n\}$. Vectors are assumed to be column vectors. $\bOnes$ is a vector with all ones. $\bI$ is the identity matrix with appropriate dimension.  $\norm{\cdot}$ is assumed to be the $2$-norm. For a matrix $\bA$, $\lambda_{max}(\bA)$ is the maximum eigenvalue of $\bA$ and $\bA^{\dagger}$ is the Moore–Penrose inverse of $\bA$. $\mO(\cdot)$ is the big-O notation.

\section{Graph Federated Learning}\label{Problem}

\begin{figure}[t!]
\centering

\includegraphics[scale = 0.3]{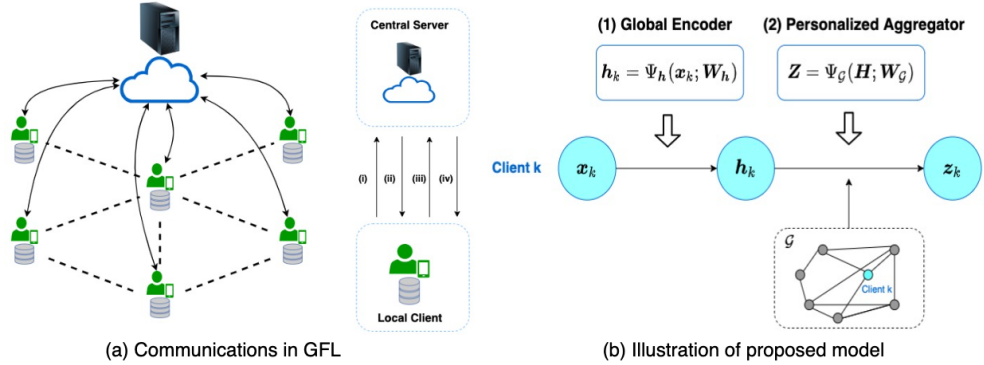}

\caption{(a) four steps in each communication round: (i) Uploading models. (ii) Broadcasting aggregated model. (iii) Uploading hidden representations. (iv) Broadcasting hidden representations. (b) global encoder $\Psi_h$ is shared since the same task is shared among clients. This is why \texttt{FedAvg} can be utilized in this multi-client system. The personalized aggregator $\Psi_{\mG}$ accounts for the statistical heterogeneity across clients.}
\label{fig: GFL_framework}
\vspace{-4mm}
\end{figure}

\subsection{Preliminaries}
\label{Sec: preliminaries}
\textbf{Federated Learning.} In typical FL, multiple clients collectively solve a global task. Our work focuses on the centralized setting with a central server, and we consider the following consensus optimization problem:
\begin{equation}
\begin{aligned}
\label{FL_problem}
    \min_{\bW} \mbox{ } F(\bW):=\frac{1}{N} \sum_{k=1}^N F_k(\bW)
\end{aligned}
\end{equation}
where $N$ is the number of clients, and $\bW$ is the model parameter. $F(\bW)$ is the global loss function and $F_k(\bW)$ is the local loss function. For client $k$, it has access only to its local data and conducts local update $\bW_{k}^{t+1} \leftarrow \bW_k^t - \eta \bg^t_k$ where $\bg^t_k := \nabla \Hat{F}_k(\bW_k^t)$ is the stochastic gradient estimator of $\nabla F_k(\bW_k^t)$ and $\eta$ is the learning rate. Throughout our work, $\nabla$ is gradient w.r.t model parameter $\bW$. Denote $I$ as the number of local updates between two communication rounds. During the FL process, after $I$ steps of local update, the central server aggregates the latest models from clients according to \texttt{FedAvg} \citep{konevcny2016federated}: $\Bar{\bW}^t = \frac{1}{N} \sum_{k=1}^N \bW_k^t$.

\textbf{Statistical Heterogeneity.} The goal of FL is to minimize the global loss on the average data distribution across clients, as shown in Eq.\eqref{FL_problem}. However, in most substantial applications of FL, clients collect data in a Non-IID distributed manner, leading to a fundamental statistical heterogeneity/shift problem in FL \citep{li2020federated, karimireddy2020scaffold}. \citep{fallah2020personalized, laguel2021superquantile} have suggested quantification of statistical heterogeneity. In this paper, we use the term "level of statistical heterogeneity" to describe how large the statistical shift is across clients.

\subsection{GFL Problem Formulation}
\label{Sec: graph_federated_optimization}
The topological structure which describes the Non-IID relationship among clients' distributions is an undirected graph denoted as $\mG = (\mV, \mE)$ where $\mV$ is a set of $N$ clients and $\mE$ is a set of edges. The adjacency matrix of $\mG$ is denoted as $\bA \in \{0,1\}^{N \times N}$. Throughout this paper, nodes in graph $\mG$ are referred to as clients. Furthermore, denote $\bXi_{\mG} = (\bX, \bY) \in \bR^{N \times (d+c)}$ as the data matrix where $\bX \in \BR^{N \times d}$ is the feature matrix with the number of features $d$ and $\bY \in \BR^{N \times c}$ is the label matrix with the dimension of label $c$. To formulate the GFL problem generally, consider $\bXi_{\mG}$ as a random matrix from a distribution $\mD_{\mG}$ which depends on $\mG$, that is, $\bXi_{\mG} \sim \mD_{\mG}$. More specifically, we define the $k$-th row vector of $\bXi_{\mG}$ as $\bxi_k:=(\bx_k, \by_k)$ where $\bx_k$ is the feature vector and $\by_k$ is a vector of labels in client $k$. Thus $\bXi_{\mG}$ is the random data matrix whose rows are correlated, and the relationship between $\bxi_k$ is described by graph $\mG$.
Here we assume that graph structure $\mG$ is deterministic. With these notations, the GFL problem is defined as:
\begin{equation}
\begin{aligned}
\label{GFL_problem}
    \min_{\bW} \mbox{ } \frac{1}{N} \sum_{k=1}^N
    F_k(\bW), \text{where } F_k(\bW) := \BE[f_k(\bW; \bXi_{\mG})],
\end{aligned}
\end{equation}
and $f_k(\bW; \bXi_{\mG})$ is the local loss after observing data matrix $\bXi_{\mG}$, indicating the local objective function of client $k$ depends not only on the data collected by the $k$-th client but also the data from other clients. This is the key difference between GFL and conventional FL. As our discussion in $\S$\ref{Introduction}, this crucial difference induces the data-sharing conflict. Therefore, we propose the following hidden representation sharing method to address the challenge of data-sharing conflict in the GFL problem.

\subsection{Hidden Representation Sharing}
\label{Sec: hidden_representation_sharing}

Our proposal is using hidden representation. The hidden representations are allowed to be shared across clients in network $\mG$, and a neighborhood aggregator is applied to these hidden representations of all nodes. For client $k$, define its hidden representation $\bh_k$ as follow
\begin{equation}
\label{hidden_representations}
\begin{aligned}
    \bh_k &= \Psi_{h}(\bx_k;\bW_{h})
\end{aligned}
\end{equation}
where $\Psi_{h}(\cdot;\bW_{h})$ is a hidden encoder such as the multi-layer perceptron (MLP) parametrized by $\bW_{h}$. The hidden representation matrix is denoted as $\bH \in \BR^{N \times d_h}$ where the $k$-th row vector of $\bH$ is $\bh_k$ and $d_h$ is the dimension of the hidden representations. Graph representations are defined by neighborhood representation aggregation of $\bH$:
\begin{equation}
\label{graph_representations}
\begin{aligned}
    \bZ &= \Psi_{\mG} (\bH; \bW_{\mG})
\end{aligned}
\end{equation}
where $\Psi_{\mG}(\cdot;\bW_{\mG})$ is a neighborhood aggregator parametrized by $\bW_{\mG}$. $\bZ \in \BR^{N \times d_z}$ is the graph representation matrix whose $k$-th row vector is denoted as $\bz_k$ and $d_z$ is the dimension of the graph representations. In most classification tasks, $d_z=c$. Model parameter $\bW = (\bW_{h}, \bW_{\mG})^{\top}$ is the concatenate of $\bW_{h}$ and $\bW_{\mG}$.  With these privacy-preserving representations, the loss function for the corresponding graph federated optimization can be expressed as,
$f_k(\bW; \bXi_{\mG}) := \ell(\by_k, \bz_k)$ where $\ell$ is the pre-specified loss function such as cross-entropy for classification task. 

\begin{remark}
The explication in Figure \ref{fig: GFL_framework} shows that the hidden encoder is a global model which facilitates the involvement of FL, while the neighborhood aggregator is the personalized model which accounts for statistical heterogeneity. Intuitively,
$\Psi_{h}$ contributes to privacy protection and representation extraction. Meanwhile, $\Psi_{\mG}$ serves as modeling the heterogeneity using graph $\mG$.
Note that if we set $\Psi_{\mG}$ as the identity mapping (ignore the graph information), our solution reduces to the conventional FL solution to learn a global model $\Psi_{h}$. In addition, when $\mG$ does not fully capture the relationship across clients, $\bW_{\mG}$ serves as weights for adjusting the neighborhood aggregation level using $\mG$. See Appendix \ref{Appendix: mechanism_GFL} for further discussion.
\end{remark}

\subsection{Gradient Estimation}
\label{Sec: gradient_estimation}

In practice, to solve the GFL problem by gradient-based methods, the unbiased stochastic gradient $\nabla f_k(\bW; \bXi_{\mG})$ of client $k$ depends on data from all nodes in the network $\mG$ ($\BE[f_k(\bW; \bXi_{\mG})] = F_k(\bW)$). However, since FL restricts the data-sharing,  $\nabla f_k(\bW; \bXi_{\mG})$ is inaccessible. Another estimation of $\nabla F_k(\bW)$ for local updates must be raised. In the proposed hidden representation sharing method, local information is exchanged as a function of $\{\bh_j, \nabla \bh_j\}_{j=1}^N$ during the interactions between clients and the central server. In other words, if the client $k$ can access $\{\bh_j, \nabla \bh_j\}_{j=1}^N$, the unbiased estimator $\nabla f_k(\bW; \bXi_{\mG})$ is accessible. Formally, with the shared hidden representations, $\nabla f_k(\bW; \bXi_{\mG})$ can be expressed as a function of hidden representations: $\nabla f_k(\bW; \bXi_{\mG}) = \phi_k (\bh_1,...,\bh_N)$.
Note that $\nabla \bh_j$ is also a function of $\bh_j$, and we consider the case that estimating $\nabla \bh_j$ is completely based on an estimator of $\bh_j$. Furthermore, define $\Hat{\bh}_{j \rightarrow k}$ as the estimation of the hidden representation $\bh_j$ for client $k$. Then the biased estimator of $\nabla F_k(\bW)$ is, 
\begin{equation}
\begin{aligned}
\label{gradient_estimator_single}
    \nabla \Hat{f}_k (\bW; \bxi_{k}) = \phi_k (\Hat{\bh}_{1 \rightarrow k},...,\Hat{\bh}_{N \rightarrow k}) \mbox{, }\forall k \in [N].
\end{aligned}
\end{equation}
The strategy to design estimator $\Hat{\bh}_{j \rightarrow k}$ depends on the concrete scenario. In $\S$\ref{Scenario}, we provide a gradient compensation strategy with theoretical analysis in Appendix \ref{Appendix: analysis_gradient_compensation} and the empirical results of this biased estimation strategy are provided in $\S$\ref{Experiments}. In practice, the estimator of $\nabla F_k(\bW)$ is the batch mean of biased stochastic gradients. Formally, suppose $\mB_k : =\{\bxi_{k,s}\}_{s=1}^{|\mB_k|}$ is the mini-batch with batch size $|\mB_k|$ for some local update in client $k$. $\nabla \Hat{f}_k (\bW; \bxi_{k,s})$ is the estimated gradient which depends on the example $\bxi_{k,s}$. The batch mean of biased stochastic gradients is defined as 
$
\nabla \Hat{F}_k(\bW)
:= \frac{1}{|\mB_k|} \sum_{s \in \mB_k}
\nabla \Hat{f}_k (\bW; \bxi_{k,s})
$.

\textbf{Privacy in GFL.} 
FL and GFL require the protection of node-level privacy: client can not share their own collected data with both other clients and the central server directly. However, directly sharing $\{\bh_j, \nabla \bh_j\}_{j=1}^N$ raises the concern about raw data recovery by untrustworthy clients or the central server. Our proposed solution does not violate node-level privacy even though we allow sharing hidden representations and the corresponding gradients during the communication between clients and the central server. By using the personalized neighborhood aggregator, clients will not receive $\{\bh_j, \nabla \bh_j\}_{j=1}^N$ directly, making the raw data recovery infeasible. In addition, the concern about the unreliable server can be addressed by applying Differential Privacy (DP) method in GFL. Detailed discussion and experiments for DP are given in Appendix \ref{Appendix: privacy_GFL} and Appendix \ref{appendix: noisy_gradient}

\subsection{Graph Federated Learning Procedure}
\label{Sec: GFL_procedure}

A framework of communications in GFL with hidden representation sharing is described in Figure \ref{fig: GFL_framework}. An concrete example is Algorithm \ref{alg: GFL-APPNP} introduced in $\S$ \ref{sec: GFL_APPNP}. Steps at each communication round are: 
\vspace{-2mm}
\begin{itemize}[leftmargin=18pt, itemsep = -0.6pt]
    \item[\textbf{(1)}] \textbf{Uploading Models:} Clients parallelly upload the latest models to the central server. 
    \item[\textbf{(2)}] \textbf{Centralizing Models:} Central server aggregates models by \texttt{FedAvg} and broadcasts the aggregated result.
    \item[\textbf{(3)}] \textbf{Uploading Hidden Representations:} Clients compute estimated hidden representation and their gradient using the received aggregated model in step (2) and then parallelly upload them to the central server.
    \item[\textbf{(4)}]  \textbf{Broadcasting Hidden Representations:} Central server allocates estimated hidden representation and their gradients and broadcasts the aggregated ones to clients.
    \item[\textbf{(5)}] \textbf{Local Updates:} Clients parallelly perform local updates for $I$ times.
\end{itemize}

\section{Theoretical Analysis}\label{Theory}

\begin{figure}[t!]
\centering
\includegraphics[scale = 0.2]{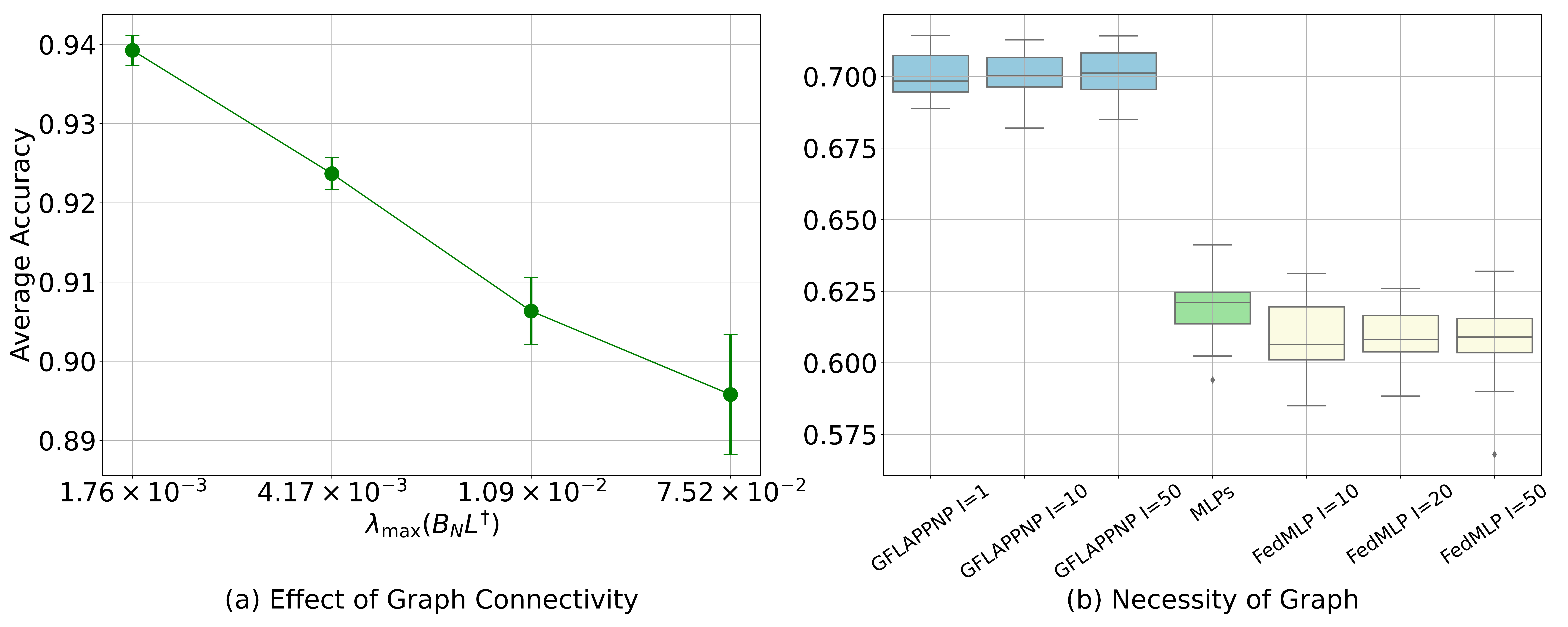}
\vspace{-1mm}
\caption{(a) the markers represent accuracy for graphs with different connectivity measured by $\lambda_{\max} (\bB_N \bL^{\dagger})$ discussed in Remark \ref{remark: theorem}. 
(b) box plot of average accuracy on $20$ synthetic graphs over our methods and baseline models.
More details are in Appendix \ref{Appendix: additional_experiments}.
}
\label{fig: graph and connectivity}
\vspace{-5mm}
\end{figure}


\subsection{Assumptions}
\label{sec:assumptions}

\begin{assumption}
\label{assumption: smoothness}
(Smoothness) Local loss function $F_k$ is differentiable and assumed to be smooth with constant $\rho_f$, $\forall k \in [N]$. Formally, $\forall \bW, \bW^{\prime}$, $\exists \rho_{f} > 0$ such that 
$
\norm{\nabla F_k(\bW) - \nabla F_k(\bW^{\prime})} 
\leq
\rho_{f} \norm{\bW - \bW^{\prime}}
$.
\end{assumption}

\begin{assumption}
\label{assumption: bound_hidden_representation_estimation}
(Bound for Hidden Representation Estimation) Hidden Representation $\bh_j$ of client $j$ is estimated by $\Hat{\bh}_{j \rightarrow k}$ for local updates at client $k$. The mean squared error of estimation is bounded in the following sense: $\forall j,k \in [N]$, $\exists \sigma_j^2 > 0$ and $\sigma_H^2 : = \sum_{j=1}^N \sigma_j^2$ such that,
$
\BE[\norm{\Hat{\bh}_{j\rightarrow k} - \bh_j }^2] \leq \sigma_j^2
$.
\end{assumption}

\begin{assumption}
\label{assumption: graph_smoothing}
(Graph Smoothing on Gradients)
Graph $\mG = (\mV, \mE)$ is connected graph and $\exists \kappa^2 > 0$ such that $\forall \bW$,
$
\sum_{(i,j)\in \mE} \norm{\nabla F_i(\bW) - \nabla F_j(\bW)}^2 \leq \kappa^2
$.
\end{assumption}

\begin{assumption}
\label{assumption: bound_stochastic_gradient}
(Bounds for Stochastic Gradient)\\
(i) (Bounded Variance) Variance of unbiased stochastic gradient $\nabla f_k (\bW; \bXi_{\mG})$ is bounded. Formally, $\exists \sigma_{\mG}^2 >0$ such that $\forall \bW$,
\begin{equation}
\begin{aligned}
\sum_{k=1}^N \BE [\norm{\nabla f_k (\bW; \bXi_{\mG}) - \nabla F_k(\bW)}^2] 
\leq \sigma_{\mG}^2.  
\end{aligned}
\end{equation}
(ii) (Smoothness) Denote $\nabla f_k (\bW; \bXi_{\mG}) = \phi_k (\bh_1,...,\bh_N)$. Assume for any $k \in [N]$, $\phi_k$ satisfies that $\forall \bh_i, \bh_i^{\prime}$ and $i \in [N]$, $\exists \rho_{\phi} > 0$ such that,
\begin{equation}
\begin{aligned}
\norm{\phi_k (\bh_1,...,\bh_N) - \phi_k (\bh_1^{\prime},...,\bh_N^{\prime})}
\leq 
\rho_{\phi} \big(\sum_{i=1}^N \norm{\bh_i - \bh_i^{\prime}}\big)^{1/2}.  
\end{aligned}
\end{equation}
\end{assumption}

\textbf{Interpretation of Assumptions.} 
(i) Assumption \ref{assumption: smoothness} is commonly assumed in the literature on nonconvex optimization and FL. 
(ii) Assumption \ref{assumption: bound_hidden_representation_estimation} is the goodness of hidden representations estimation. $\sigma_j^2$ represents the estimation error of $\Hat{\bh}_{j \rightarrow k}$ and $\sigma_H^2$ quantifies the total estimation error.
(iii)  $\kappa^2$ in Assumption \ref{assumption: graph_smoothing} quantifies this statistical heterogeneity among clients by considering the network structure which captures the relationship among clients' distributions.  A previous work \citep{fallah2020personalized} shows that there is a connection between data distribution shift among clients and the gradient shift among clients. Lemma \ref{appendix: gradient_graph_smoothing} provides the upper bound for the level of statistical heterogeneity based on $\kappa^2$ and connectivity of $\mG$. 
(iv) Assumption \ref{assumption: bound_stochastic_gradient} ensures $\nabla f_k (\bW; \bXi_{\mG})$ satisfies two properties. First, it has a bounded variance ($\sigma_{\mG}^2$) which is commonly presumed in previous works. The second one is a sense of smoothness of $\nabla f_k (\bW; \bXi_{\mG})$ in terms of a function of hidden representations with smoothness quantified by $\rho_{\phi}$.

\subsection{Convergence Analysis}
\label{sec: convergence}

\begin{theorem}
\label{theorem: nonconvex_main}
Consider GFL optimization problem \eqref{GFL_problem} under Assumptions \ref{assumption: smoothness}, \ref{assumption: bound_hidden_representation_estimation}, \ref{assumption: graph_smoothing} and \ref{assumption: bound_stochastic_gradient}. Use the federated learning procedure described in $\S$ \ref{Problem}. Suppose $\eta$ and $I$ satisfies $I \eta^2 < 1/13 \rho_f^2$, then for all $T \geq 1$, we have
\begin{equation}
\begin{aligned}
    \frac{1}{T} 
    \sum_{t=0}^{T-1} \BE[\norm{\nabla F(\Bar{\bW}^t)}^2]
    \leq 
    \mO(\frac{1}{\eta T}) + 
    \mO(\frac{I^2 \eta^2 \sigma_{\mG}^2 }{N}) +
    \mO(I^2 \eta^2 \sigma_H^2) 
    +
    \mO(\frac{I^2 \eta^2 \kappa^2 \lambda_{\max} (\bB_N \bL^{\dagger})}{N})
\end{aligned}
\end{equation}
Where $\bB_N :=  \frac{1}{N} \bI - \frac{1}{N^2}\bOnes \bOnes^{\top}$ and $\bL$ is the Laplacian matrix of $\mG$.
\end{theorem}
\begin{proof}
See Appendix \ref{appendix : proof_nonconvex_main}.
\end{proof}

\begin{corollary}
\label{corollary: rate}
Under the setting of Theorem \ref{theorem: nonconvex_main}. Suppose learning rate is chosen as $\eta = \frac{\sqrt{N}}{\sqrt{T}}$ and removing smoothness constants $\rho_f$ and $\rho_{\phi}$, we have
\begin{equation}
\begin{aligned}
    \frac{1}{T} 
    \sum_{t=0}^{T-1} \BE[\norm{\nabla F(\Bar{\bW}^t)}^2] 
    = 
    \mO(\frac{1}{\sqrt{N T}}) + \mO(\frac{N I^2}{T}) + \mO(\frac{\lambda_{\max} (\bB_N \bL^{\dagger}) I^2}{T})
\end{aligned}
\end{equation}
\end{corollary}
\begin{remark}
\label{remark: theorem} According to our Theorem \ref{theorem: nonconvex_main} and Corollary \ref{corollary: rate}, when $\sigma_H^2=0$, that is, ignoring the node-level privacy issue and access to the unbiased stochastic gradient, our convergence result matches the rate of the previous works \citep{jiang2018linear, yu2019parallel, stich2018local} with the gradient deviation among clients is described by a graph structure.
For the effect from graph structure, since $\lambda_{\max} (\bB_N \bL^{\dagger})$ is an indication of the connectivity of graph $\mG$ with normalized by averaged aggregation $\bB_N$ (large $\lambda_{\max} (\bB_N \bL^{\dagger})$ means a bad connectivity and high level of statistical heterogeneity), we can expect that a graph with good connectivity ensures a better performances as shown in Figure \ref{fig: graph and connectivity}. This observation matches our intuition that a smaller level of statistical heterogeneity in FL secures a better performance. Moreover, our Corollary \ref{corollary: rate} keep the linear speed up (w.r.t number of clients) when $I = 1$ \citep{yu2019linear}.
\end{remark}

\section{\texttt{GFL-APPNP} for Classification}\label{Scenario}

\begin{algorithm}[t!]
\caption{GFL-APPNP for Classification}
\label{alg: GFL-APPNP}
\begin{algorithmic}
\Require Initialize $\{\bW^0_k\}_{k=1}^N$, $\eta$, $T$, $I$. Compute $\bTildeA$.
\For{$t = 0,...,T-1$}
    \If{$t \mod I = 0$}
        \State \textbf{On Client $k \in [N]$ Parallelly}:
        \vspace{1mm}
        \Indent
            \State Uploads latest model $\bW^{t}_k$.
        \EndIndent
        \State \textbf{On Central Server}:
        \vspace{1mm}
        \Indent
            \State Broadcast $\Bar{\bW}^{t} = \frac{1}{N} \sum_{k=1}^N \bW^{t}_k$ to all clients.
        \EndIndent
        \State \textbf{On Client $j \in [N]$ Parallelly}:
        \vspace{1mm}
        \Indent
            \State Compute $\Hat{\bh}_j= \frac{1}{n_j} \sum_{s=1}^{n_j} \MLP(\bx_{j,s};\Bar{\bW}^{t})$ and set $\bW^t_j = \Bar{\bW}^{t}$.
            \State Upload $\Hat{\bh}_j$ and $\nabla \Hat{\bh}_{j}$.
        \EndIndent
        \State \textbf{On Central Server}:
        \vspace{1mm}
        \Indent
            \State Compute $\bC_k = \sum_{j \neq k} \bTildeA_{kj} \Hat{\bh}_j$ and $\nabla \bC_k = \sum_{j \neq k} \bTildeA_{kj} \nabla \Hat{\bh}_j$, $\mbox{} \forall k \in [N]$.
            \State Broadcast $\bC_j$ and $\nabla \bC_j$ to client $j$ for all $j \in [N]$.
        \EndIndent
    \EndIf
    \State \textbf{On Client $k \in [N]$ Parallelly}:
    \vspace{1mm}
    \Indent
        \State Compute $\bh_{k,s} = \MLP(\bx_{k,s}; \bW^t_k)$ and $\Hat{\by}_{k,s} = \Softmax(\bTildeA_{kk} \bh_{k,s} + \bC_k)$, $\mbox{} \forall s \in \mB_k$
        \State Compute $\nabla \Hat{f}_k (\bW^t_k; \bxi_{k,s})  = (\by_k - \Hat{\by}_{k,s}) (\bTildeA_{kk} \nabla \bh_{k,s} + \nabla \bC_k) \mbox{, } \forall s \in \mB_k$.
        \State Compute $\bg^{t}_k = \frac{1}{|\mB_k|} \sum_{s \in \mB_k} \nabla \Hat{f}_k (\bW^t_k; \bxi_{k,s})$.
        \State Compute $\bW^{t+1}_k \leftarrow \bW^{t}_k - \eta \bg^{t}_k$.
    \EndIndent
\EndFor
\end{algorithmic}
\end{algorithm}

\subsection{GFL for Classification Tasks on Graphs}
\label{sec: clissification_tasks_graph}

\textbf{Deterministic Node Classification (DNC).} Graph-based semi-supervised node classification is the most popular classification task on graphs. In this paper, we call it deterministic node classification since $\bxi_k$ for each node is deterministic with one feature vector and one label. The GFL problem can formulate this task in $\S$ \ref{Problem} by assuming $\bxi_k$ is from a degenerated distribution.

\textbf{Stochastic Node Classification (SNC).} An extended version of the deterministic node classification is the setting where local distribution at each node is not degenerated, namely stochastic node classification. This task is a semi-supervised node classification that classifies the nodes by learning from local distributions. Similarly, this task can be formulated by the GFL problem in $\S$ \ref{Problem} by assuming the randomness of $\bXi_{\mG}$ is only from $\bX$. An important real-world application for this task is the user demographic label prediction in social networks.  

\textbf{Supervised Classification (SC).} Consider the supervised learning task on clients, which is another classification task on graphs assuming the label of a client also follows a distribution. We call it supervised classification. The objective of this task is to classify the feature vectors in all clients. This task assumes the randomness of $\bXi_{\mG}$ is from both $\bX$ and $\bY$, resulting in that each client might have examples with different labels. One practical application is the patient classification problem in hospitals with insufficient medical records. 

More details about classification tasks are provided in Appendix \ref{Appendix: experiment_task}. Moreover, our GFL setting introduced in $\S$ \ref{Problem} is not only for standard supervised learning but also can be easily extended to semi-supervised client classification like DNC and SNC.

\subsection{\texttt{GFL-APPNP} Algorithm}
\label{sec: GFL_APPNP}

Approximate Personalized Propagation of Neural Predictions (\texttt{APPNP}) \citep{klicpera2018predict} is one of the state-of-the-art GNN models. With the notations and context in Section \ref{Problem}, \texttt{APPNP} has the hidden encoder $\Psi_h$ and neighborhood aggregator $\Psi_{\mG}$ defined as follow,
\begin{subequations}
\label{APPNP}
\begin{align}
    \bh_k &= \Psi_h (\bx_k;\bW)=  \MLP(\bx_k;\bW), \\
    \bZ &= \sum_{i=0}^M (1 - \alpha I\{i\neq M\}) \alpha^i (\Hat{\bD}^{-1/2} \bHatA \Hat{\bD}^{-1/2})^i \bH
    =
    \bTildeA \bH,
\end{align}
\end{subequations}
where $\alpha$ is teleport probability \citep{klicpera2018predict} and $M$ is the total steps for personalized propagation. $\bHatA$ is the adjacency matrix with self-loop, and $\Hat{\bD}$ is the degree matrix with self-loop. $\bTildeA$ is defined as $\bTildeA:= \sum_{i=0}^M (1 - \alpha I\{i\neq M\}) \alpha^i (\Hat{\bD}^{-1/2} \bHatA \Hat{\bD}^{-1/2})^i$. It can be interpreted as the "adjacency matrix" after $M$ steps random walk, which shows the reachability between two nodes in the structure after propagations. Loss function $\ell$ in \texttt{APPNP} is the cross entropy loss. In \texttt{APPNP}, $\bW_h$ discussed in Eq.\eqref{hidden_representations} refers to $\bW$ since the neighborhood aggregator in \texttt{APPNP} is not parametrized. Note that original \texttt{APPNP} is proposed to solve the deterministic node classification task. Denote predicted one-hot vectors as $\Hat{\bY} = \Softmax(\bZ)$ where $\Softmax(\cdot)$ is a row-wise softmax function. The gradients can be expressed explicitly as follow,
\begin{equation}
\begin{aligned}
\nabla f_k (\bW; \bXi_{\mG}) 
= (\by_k - \Hat{\by}_k) \sum_{i=1}^N \bTildeA_{ki} \nabla \bh_i,
\end{aligned}
\end{equation}
where $\by_k$ is the one-hot vector for the true label of client $k$ and $\Hat{\by}_k$ is the predicted probability vector for the label of client $k$. $\bTildeA_{ki}$ is the element in matrix $\bTildeA$. Note that the hidden representation sharing contributes to two parts in the gradient for local loss $f_k$: one is $\Hat{\by}_k$ and the other one is $\{\nabla \bh_i\}_{i \neq k}$. A good property of using personalized propagation as the neighborhood aggregator is the linearity, which means
\begin{subequations}
\begin{align}
    \Hat{\by}_k & 
    = \Softmax(\bTildeA_{kk} \bh_k + \bC_k), \\
    \nabla f_k (\bW; \bXi_{\mG}) 
    & =
    (\by_k - \Hat{\by}_k) (\bTildeA_{kk} \nabla \bh_k + \nabla \bC_k),
\end{align}
\end{subequations}
where $\bC_k := \sum_{i \neq k} \bTildeA_{ki} \bh_i$ and $\nabla \bC_k:= \sum_{i \neq k} \bTildeA_{ki} \nabla \bh_i$. Clearly, $\bC_k$ and its gradient $\nabla \bC_k$ are aggregated information for client $k$. Therefore, in practice, the central server only needs to broadcast $\bC_k$ and its gradient to client $k$ in the communication round, which provides private communication since the hidden representations are not shared directly. 

We propose \texttt{GFL-APPNP} algorithm for GFL problem on classification tasks, using hidden representation sharing. In addition, we use the latest aggregated model to compute hidden representations at each communication round as our gradient compensation strategy. As a concrete example, consider client $j$ has $n_j$ local feature vectors $\{\bx_{j,s}\}_{s=1}^{n_j}$,
suppose $t_0 < t$ is the largest multiple of $I$,
\begin{equation}
\label{gradient_compensation}
\begin{aligned}
\Hat{\bh}^t_{j \rightarrow k}
=
\begin{cases}
\Psi_{h} (\bx_{k}; \bW^t_{h,k}) & j=k\\
\frac{1}{n_j} \sum_{s=1}^{n_j} \Psi_{h}(\bx_{j,s};\Bar{\bW}^{t_0}_h)
&  j \neq k
\end{cases},
\end{aligned}
\end{equation} 
where $\Hat{\bh}^t_{j \rightarrow k}$ is the estimation for $\bh_{j \rightarrow k}$ at time $t$. Our compensation strategy satisfies the guarantee discussed in Assumption \ref{assumption: bound_hidden_representation_estimation} with additional assumptions. See Appendix \ref{Appendix: analysis_gradient_compensation} for detailed discussion for gradient compensation. Summary of \texttt{GFL-APPNP} is described in Algorithm \ref{alg: GFL-APPNP}.  Our proposed \texttt{GFL-APPNP} is a FL version for \texttt{APPNP}, which fulfills the FL for a GNN model. It is noteworthy that \texttt{GFL-APPNP} with $I=1$ is \textit{equivalent} to the FL for the vanilla \texttt{APPNP} for deterministic node classification tasks.

\section{Experiments}\label{Experiments}

\begin{table}[t!]
\caption{\footnotesize Results on the deterministic node classification task. This table provides the results of the average test accuracy and the corresponding $95\%$ confidence interval. "SG" represents synthetic graphs, $\texttt{SAGE}$ represents \texttt{GraphSAGE}, and $\texttt{GAP}$ represents \texttt{GFL-APPNP}.}
\label{table: DNC_results}
\small
\centering

\begin{tabular}{p{1cm} p{1.3cm}p{1.3cm}p{1.3cm}p{1.3cm}p{1.3cm}p{1.3cm}p{1.3cm}} 
\hline
& \texttt{$\text{GAP}_{I=10}$} & \texttt{$\text{GAP}_{I=20}$} & \texttt{$\text{GAP}_{I=50}$} & \texttt{APPNP} & \texttt{GCN} & \texttt{GAT} & \texttt{SAGE}\\
\hline
SG & $93.4 (0.99)$ & $93.3 (0.94)$ & $93.0(0.96)$ & $93.2(0.92)$ & $\textbf{95.2(0.54)}$ & $93.3(1.03)$ & $70.2(4.21)$\\
SubCora & $54.1(3.72)$ & $\textbf{54.3(3.73)}$ & $54.0(3.73)$ & $54.2(3.69)$ & $51.9(3.78)$ & $47.9 (3.01)$ & $47.0 (3.73)$\\
\hline
\end{tabular}
\vspace{-1.5mm}
\end{table}

\begin{figure*}[t!]
\centering
\includegraphics[scale = 0.2]{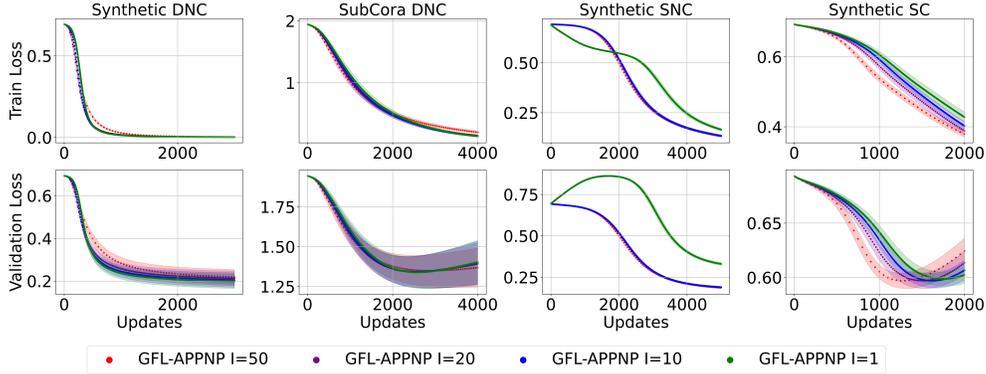}
\caption{Train and validation loss for DNC (first column), subCora (second column), SNC (third column), and SC  (fourth column). For different lines, the numbers of points are different given the same number of updates 
(if $T=3000$ and $I=10$, $3000/10 = 300$ points are in the line).  
The shaded area represents $95\%$ CIs. See Appendix \ref{appendix: experiments} for more variants.}
\label{fig: main_figure}
\end{figure*}


\subsection{Deterministic Node Classification}
\label{experiment: deterministic_node_classification}

We compare the proposed \texttt{GFL-APPNP} to baseline models including \texttt{GCN}, \texttt{GAT}, and \texttt{GraphSAGE} under the DNC setting described in $\S$\ref{sec: clissification_tasks_graph}.
For synthetic data, we use contextual Stochastic Block Models (cSBMs) \citep{deshpande2018contextual} to generate synthetic graphs with approximately two equal-size classes. For real-world data, we use subgraphs of Cora \cite{sen2008collective}, namely subCora, due to the limitation of computational resources. The details of the generation of synthetic graphs by cSBMs and the generation of subCora graphs are provided in Appendix \ref{Appendix: data_generation}.
For the proposed \texttt{GFL-APPNP}, we use a two-layer \texttt{MLP} with $64$ hidden units. $\alpha$ is chosen to be $0.1$ and the total number of steps for personalized propagation $M$ is set as $10$, following the same configuration as it in the \citep{chien2020adaptive} for the \texttt{APPNP} model. $I$ is set to be $\{1, 10, 20, 50\}$. SGD is applied as our optimizer with the optimized learning rate. Baseline models including \texttt{GCN}, \texttt{GAT}, and \texttt{GraphSAGE} follow the same design of the well-optimized hyperparameters from \citep{kipf2016semi, velivckovic2017graph, hamilton2017inductive}. The details for all models are provided in Appendix \ref{Appendix: model_description}.
Table \ref{table: DNC_results} and the first two columns of Figure \ref{fig: main_figure} show that our method of different $I$ can match the performance of the vanilla \texttt{APPNP} on both synthetic graphs and subCora graphs. Our method rivals baseline models based on Table \ref{table: DNC_results}.

\subsection{Stochastic Node Classification}
\label{experiment: stochastic_node_classification}

We also conduct experiments under the SNC setting described in $\S$\ref{sec: clissification_tasks_graph} to test the robustness of \texttt{GFL-APPNP}. As we know, current graph machine learning models are not designed for the SNC task. Therefore our experiment will focus on the proposed \texttt{GFL-APPNP}.
The original cSBMs can not be used to generate data for this task since they are not designed to generate graphs whose nodes have multiple features. We modify the original cSBMs to generate distributions for each client. Details about the modifications for the SNC task are available in Appendix \ref{Appendix: data_generation}. Similar to $\S$\ref{experiment: deterministic_node_classification}, we use the same hyperparameters and the same number of updates. The third column of Figure \ref{fig: main_figure} and additional Table \ref{table: SNC_results} provided in Appendix \ref{Appendix: additional_experiments} show that our method is valid for solving the SNC task.




\subsection{Supervised Classification }
\label{experiment: supervised_classification}

We compare \texttt{GFL-APPNP}, \texttt{MLPs}, and \texttt{FedMLP} under the SC setting described in $\S$\ref{sec: clissification_tasks_graph}. Similar to the SNC task, we modify the original cSBMs to generate distributions for each client. Details about the modifications for the SC task are given in Appendix \ref{Appendix: data_generation} as well. Both baseline models \texttt{MLPs} and \texttt{FedMLP} share the same model structure, a two-layer \texttt{MLP} model with $64$ hidden units. Similar to $\S$\ref{experiment: deterministic_node_classification}, for the proposed \texttt{GFL-APPNP}, we use the same hyperparameters and the same numbers of updates. More details and results are provided in Appendix \ref{appendix: experiments}.
The last column of Figure \ref{fig: main_figure} and Table \ref{table: table_all} provided in Appendix \ref{Appendix: additional_experiments} shows that our method is valid for solving the SC task and it demonstrates the necessity of graph in GFL problem as shown in Figure \ref{fig: graph and connectivity}.

\section{Conclusion}\label{Conclusion}

In this paper, we formulate Graph Federated Learning on multi-client systems.
To tackle the fundamental data sharing conflict between LoG and FL, we propose an FL solution with hidden representation sharing.
In theory, we provide a non-convex convergence analysis.
Empirically, by experimenting with several classification tasks on graphs, we validate the proposed method on both real-world and synthetic data. Our experimental results show that the proposed method provides an FL solution for GNNs and works for different practical tasks on graphs with a competitive performance that matches our theory.


\bibliographystyle{unsrtnat}
\bibliography{reference}

\clearpage

\appendix

\section{Proofs of Main Results}\label{appendix: main}

\subsection{Basic Algebra}
\begin{lemma}
\label{appendix: Basic_algebra}
Let $\ba_1,...,\ba_N$ be any $N$ real vectors, then
\begin{equation*}
\begin{aligned}
    & (i) \quad 
    \ba_{1}^{\top} \ba_{2}  = \frac{1}{2}(\norm{\ba_1}^2 + \norm{\ba_2}^2 - \norm{\ba_1 - \ba_2}^2)  \leq \frac{1}{2} ( \norm{\ba_{1}}^2 + \norm{\ba_{2}}^2) \\
    & (ii) \quad 
    \frac{1}{N} \sum_{k=1}^N \norm{\ba_k - \frac{1}{N} \sum_{j=1}^N \ba_j}^2 \leq \frac{1}{N} \sum_{k=1}^N \norm{\ba_k}^2 \\
    & (iii) \quad
    \norm{\sum_{k=1}^N \ba_k}^2 \leq N \sum_{k=1}^N \norm{\ba_k}^2 
\end{aligned}
\end{equation*}
\end{lemma}

\subsection{Proof of Theorem \ref{theorem: nonconvex_main}}
\label{appendix : proof_nonconvex_main}

\begin{proof}
Note that even though the aggregated model $\Bar{\bW}^t$ are not computed at each time $t$, there is virtual averaged model $\Bar{\bW}^t$ whose corresponding averaged update is:
\begin{equation}
\begin{aligned}
\label{virtual_update}
    \Bar{\bW}^{t+1} \leftarrow \Bar{\bW}^t - \eta \Bar{\bg}^t
\end{aligned}
\end{equation}
where
\begin{equation}
\begin{aligned}
    \Bar{\bg}^t := \frac{1}{N} \sum_{k=1}^N \bg^t_k.
\end{aligned}
\end{equation}
Technical Lemmas for this proof are listed as follow,
\begin{lemma}
\label{appendix: mixed_variation}
\begin{equation}
\begin{aligned}
    \sum_{k=1}^N \norm{ \nabla F_k(\bW_k) - \frac{1}{N} \sum_{j=1}^N \nabla F_j(\bW_j) }^2 \leq
    6 \rho_{f}^2 \sum_{k=1}^N \norm{\bW_k - \Bar{\bW}}^2 + 3 \sum_{k=1}^N \norm{ \nabla F_k(\Bar{\bW}) - \frac{1}{N} \sum_{j=1}^N \nabla F_j(\Bar{\bW})
    }^2
\end{aligned}
\end{equation}
\end{lemma}
\begin{proof}
See Appendix \ref{appendix: proof_mixed_variation}
\end{proof}

\begin{lemma}
\label{appendix: variance_bias}
\begin{equation}
\begin{aligned}
    \sum_{k=1}^N \BE[\norm{ \nabla F_k(\bW^t_k) - \bg^t_k }^2] \leq  
    \sigma_{\mG}^2 + N \rho_{\phi}^2 \sigma_H^2
\end{aligned}
\end{equation}
\end{lemma}
\begin{proof}
See Appendix \ref{appendix: proof_variance_bias}
\end{proof}

\begin{lemma}
\label{appendix: gradient_graph_smoothing}
\begin{equation}
\begin{aligned}
    \sum_{k=1}^N \norm{ \nabla F_k(\bW) - \frac{1}{N} \sum_{i=1}^N \nabla F_i(\bW)
    }^2 \leq \kappa^2 \lambda_{\max} (\bB_N \bL^{\dagger})
\end{aligned}
\end{equation}
Where $\bB_N :=  \frac{1}{N} \bI - \frac{1}{N^2}\bOnes \bOnes^{\top}$ and $\bL$ is the Laplacian matrix of $\mG$.
\end{lemma}
\begin{proof}
See Appendix \ref{appendix: proof_gradient_graph_smoothing}
\end{proof}

\begin{lemma}
\label{appendix: between_variation}
Suppose $I$ and $\eta$ satisfy $I \eta^2 < 1/13 \rho_{f}^2$, then 
\begin{equation}
\begin{aligned}
     \frac{1}{N} \sum_{t=0}^{T-1} \sum_{k=1}^N \BE[\norm{\Bar{\bW}^t - \bW^t_k}^2] 
     & \leq 
     \frac{26}{N} T I^2 \eta^2  \sigma_{\mG}^2 + 26 T I^2 \eta^2 \rho_{\phi}^2 \sigma_H^2
     +
     \frac{65}{N} T I^2 \eta^2 \kappa^2 \lambda_{\max} (\bB \bL^{\dagger})  
\end{aligned}
\end{equation}
\end{lemma}
\begin{proof}
See Appendix \ref{appendix: proof_between_variation}
\end{proof}
\begin{remark}
Lemma \ref{appendix: mixed_variation} provides an upper bound for "mixed variation" which quantifies the variation from gradient at client $k$ with its local model and the averaged gradient of all clients with their own models. This variation is bounded by the sum of two types of variations: (i) Shift among the local models. (ii) Shift between the gradients among clients at the same model.   
Lemma \ref{appendix: variance_bias} gives the goodness for gradient estimation. This upper bound shows that the estimation error is basically from two types: (i) Variance of the unbiased estimator ($\sigma_{\mG}^2$). (ii) Mean squared error of the hidden representation estimation ($\sigma_H^2$).
Lemma \ref{appendix: gradient_graph_smoothing} is about the gradient deviation among working clients on graph $\mG$. This upper bound quantifies the level of statistical heterogeneity among clients in FL.
Combining previous Lemmas, the expected total shift among the local models at all times is controlled in Lemma \ref{appendix: between_variation}. 
\end{remark}

\textbf{Main proof of Theorem \ref{theorem: nonconvex_main}}:
\\
By $\rho_{f}$-smoothness,
\begin{equation}
\begin{aligned}
    \BE[F(\Bar{\bW}^{t+1})] \leq \BE [F(\Bar{\bW}^t) + \nabla F(\Bar{\bW}^t)^{\top}(\Bar{\bW}^{t+1} - \Bar{\bW}^t) +\frac{\rho_{f}}{2} \norm{\Bar{\bW}^{t+1} - \Bar{\bW}^t}^2 ] 
\end{aligned}
\end{equation}

For the cross term $\BE[\nabla F(\Bar{\bW}^t)^{\top}(\Bar{\bW}^{t+1} - \Bar{\bW}^t)]$, 
\begin{subequations}
\begin{align}
    &
    \BE[\nabla F(\Bar{\bW}^t)^{\top}(\Bar{\bW}^{t+1} - \Bar{\bW}^t)] \\
    =
    & -\eta \BE[\nabla F(\Bar{\bW}^t)^{\top} \Bar{\bg}^t] 
    \text{ (by Eq.\eqref{virtual_update})} \\
    =
    & -\eta \Bigg[\frac{1}{2} \BE [\norm{\nabla F(\Bar{\bW}^t)}^2] + \frac{1}{2} \BE[\norm{\Bar{\bg}^t}^2] -\frac{1}{2} \BE[\norm{\nabla F(\Bar{\bW}^t) - \Bar{\bg}^t}^2 ] \Bigg] 
    \text{ (by Lemma \ref{appendix: Basic_algebra})} \\
    = 
    & -\frac{\eta}{2} \BE [\norm{\nabla F(\Bar{\bW}^t)}^2] 
    -\frac{\eta}{2} \BE[\norm{\Bar{\bg}^t}^2] \notag \\ 
    & +\frac{\eta}{2} \BE \Bigg[\norm{\nabla F(\Bar{\bW}^t) - \frac{1}{N} \sum_{k=1}^N \nabla F_k(\bW^t_k) + \frac{1}{N} \sum_{k=1}^N \nabla F_k(\bW^t_k) - \frac{1}{N} \sum_{k=1}^N \bg^t_k }^2 \Bigg] \\
    \leq 
    & -\frac{\eta}{2} \BE [\norm{\nabla F(\Bar{\bW}^t)}^2] 
    -\frac{\eta}{2} \BE[\norm{\Bar{\bg}^t}^2] 
    + \eta \BE[\norm{\nabla F(\Bar{\bW}^t) - \frac{1}{N} \sum_{k=1}^N \nabla F_k(\bW^t_k)}^2] \notag \\
    & + \eta \BE[\norm{\frac{1}{N} \sum_{k=1}^N \nabla F_k(\bW^t_k) - \frac{1}{N} \sum_{k=1}^N \bg^t_k}^2]
    \text{ (by Lemma \ref{appendix: Basic_algebra})} \\
    \leq 
    & -\frac{\eta}{2} \BE [\norm{\nabla F(\Bar{\bW}^t)}^2] 
    -\frac{\eta}{2} \BE[\norm{\Bar{\bg}^t}^2] 
    + \frac{1}{N} \eta \rho_{f}^2 \sum_{k=1}^N \BE[\norm{\Bar{\bW}^t - \bW^t_k}^2] \notag \\
    & + \frac{1}{N} \eta \BE[\norm{ \nabla F_k(\bW^t_k) - \bg^t_k }^2] 
    \text{ (by Assumption \ref{assumption: smoothness} and Lemma \ref{appendix: Basic_algebra})}
\end{align}
\end{subequations}
Thus, plug in the above result in the smoothness,
\begin{subequations}
\begin{align}
    & \BE[F(\Bar{\bW}^{t+1}) - F(\Bar{\bW}^t) ] \\
    \leq
    & -\frac{\eta}{2} \BE [\norm{\nabla F(\Bar{\bW}^t)}^2] 
    -\frac{\eta}{2} \BE[\norm{\Bar{\bg}^t}^2] + \frac{1}{N} \eta \rho_{f}^2 \sum_{k=1}^N \BE[\norm{\Bar{\bW}^t - \bW^t_k}^2] \notag \\
    & + \frac{1}{N} \eta \sum_{k=1}^N \BE[\norm{ \nabla F_k(\bW^t_k) - \bg^t_k }^2] +\frac{1}{2} \eta^2 \rho_{f} \BE[\norm{\Bar{\bg}^t}^2] \\
    = 
    & -\frac{\eta}{2} \BE [\norm{\nabla F(\Bar{\bW}^t)}^2]
    +\frac{1}{2}\eta (\eta \rho_{f} - 1) \BE[\norm{\Bar{\bg}^t}^2] \notag \\
    & + \frac{1}{N} \eta \rho_{f}^2 \sum_{k=1}^N \BE[ \norm{\Bar{\bW}^t - \bW^t_k }^2] + \frac{1}{N} \eta \sum_{k=1}^N \BE[\norm{ \nabla F_k(\bW^t_k) - \bg^t_k }^2]
\end{align}
\end{subequations}
After re-arranging,
\begin{equation}
\begin{aligned}
    \BE[\norm{\nabla F(\Bar{\bW}^t)}^2] 
    \leq 
    & 
    \frac{2}{\eta} \BE[F(\Bar{\bW}^{t+1}) - F(\Bar{\bW}^t) ] 
    + (\eta \rho_{f} - 1) \BE[\norm{\Bar{\bg}^t}^2] \\
    + &
    \frac{2}{N} \rho_{f}^2 \sum_{k=1}^N \BE[\norm{\Bar{\bW}^t - \bW^t_k}^2] + \frac{2}{N} \sum_{k=1}^N \BE[\norm{ \nabla F_k(\bW^t_k) - \bg^t_k }^2]
\end{aligned}
\end{equation}
Now summing over $t = 0, ..., T-1$, 
\begin{subequations}
\begin{align}
    &
    \sum_{t=0}^{T-1} \BE[\norm{\nabla F(\Bar{\bW}^t)}^2] \\
    \leq 
    & \frac{2}{\eta} \BE[F(\Bar{\bW}^{0}) - F(\Bar{\bW}^t) ] + (\eta \rho_{f} - 1) \sum_{t=0}^{T-1} \BE[\norm{\Bar{\bg}^t}^2] \notag \\
    + 
    & \frac{2}{N} \rho_{f}^2 \sum_{t=0}^{T-1} \sum_{k=1}^N \BE[\norm{\Bar{\bW}^t - \bW^t_k}^2] 
    + 
    \frac{2}{N} \sum_{t=0}^{T-1} \sum_{k=1}^N \BE[\norm{ \nabla F_k(\bW^t_k) - \bg^t_k }^2] \\
    \leq
    & \frac{2}{\eta} \BE[F(\Bar{\bW}^{0}) - F(\Bar{\bW}^t) ] + (\eta \rho_{f} - 1) \sum_{t=0}^{T-1} \BE[\norm{\Bar{\bg}^t}^2] 
    +
    \frac{52}{N} T I^2 \eta^2 \rho_{f}^2 \sigma_{\mG}^2 + 52 T I^2 \eta^2 \rho_{f}^2 \rho_{\phi}^2 \sigma_H^2 \notag \\
    +
    & 
    \frac{130}{N} T I^2 \eta^2 \rho_{f}^2 \kappa^2 \lambda_{\max} (\bB_N \bL^{\dagger})
    +
    \frac{2}{N} \sum_{t=0}^{T-1} \sum_{k=1}^N \BE[\norm{ \nabla F_k(\bW^t_k) - \bg^t_k }^2]
    \text{ (by Lemma \ref{appendix: between_variation})} \\
    \leq
    & \frac{2}{\eta} \BE[F(\Bar{\bW}^{0}) - F(\Bar{\bW}^t) ] + (\eta \rho_{f} - 1) \sum_{t=0}^{T-1} \BE[\norm{\Bar{\bg}^t}^2] 
    +
    \frac{52}{N} T I^2 \eta^2 \rho_{f}^2 \sigma_{\mG}^2 + 52 T I^2 \eta^2 \rho_{f}^2 \rho_{\phi}^2 \sigma_H^2 \notag \\
    +
    & 
    \frac{130}{N} T I^2 \eta^2 \rho_{f}^2 \kappa^2 \lambda_{\max} (\bB_N \bL^{\dagger})
    +
    \frac{2}{N} T \sigma_{\mG}^2 + 2 T \rho_{\phi}^2 \sigma_H^2
    \text{ (by Lemma \ref{appendix: variance_bias})} \\
    =
    & \frac{2}{\eta} \BE[F(\Bar{\bW}^{0}) - F(\Bar{\bW}^t) ] + (\eta \rho_{f} - 1) \sum_{t=0}^{T-1} \BE[\norm{\Bar{\bg}^t}^2] 
    +
    \frac{1}{N} (52 I^2 \eta^2 \rho_{f}^2 + 2)T \sigma_{\mG}^2 \notag \\
    +
    &
    (52  I^2 \eta^2 \rho_{f}^2  + 2) T \rho_{\phi}^2 \sigma_H^2 
    +
    \frac{130}{N} T I^2 \eta^2 \rho_{f}^2 \kappa^2 \lambda_{\max} (\bB_N \bL^{\dagger})
\end{align}
\end{subequations}
Assume the optimal solution to problem \eqref{GFL_problem} exists and denoted it as $\bW^{*}$. By observing that $I \eta^2 < 1/13 \rho_{f}^2$ indicates $\eta < \frac{1}{\rho_{f}}$,
\begin{subequations}
\begin{align}
    & \frac{1}{T} 
    \sum_{t=0}^{T-1} \BE[\norm{\nabla F(\Bar{\bW}^t)}^2] \\
    \leq 
    & \frac{2}{\eta T} \BE[F(\Bar{\bW}^{0}) - F(\Bar{\bW}^t) ] 
    +
    \frac{1}{N} (52 I^2 \eta^2 \rho_{f}^2 + 2) \sigma_{\mG}^2
    + 
    (52  I^2 \eta^2 \rho_{f}^2  + 2) \rho_{\phi}^2 \sigma_H^2
    +
    \frac{130}{N} I^2 \eta^2 \rho_{f}^2 \kappa^2 \lambda_{\max} (\bB_N \bL^{\dagger}) \\
    = 
    & \frac{2}{\eta T} (F(\Bar{\bW}^{0}) - F(\bW^{*}))
    +
    \frac{1}{N} (52 I^2 \eta^2 \rho_{f}^2 + 2) \sigma_{\mG}^2
    + 
    (52  I^2 \eta^2 \rho_{f}^2  + 2) \rho_{\phi}^2 \sigma_H^2
    +
    \frac{130}{N} I^2 \eta^2 \rho_{f}^2 \kappa^2 \lambda_{\max} (\bB_N \bL^{\dagger}) \\
    = &
    \mO(\frac{1}{\eta T}) + 
    \mO(\frac{I^2 \eta^2 \sigma_{\mG}^2 }{N}) +
    \mO(I^2 \eta^2 \sigma_H^2) + 
    \mO(\frac{I^2 \eta^2 \kappa^2 \lambda_{\max} (\bB_N \bL^{\dagger})}{N})
\end{align}
\end{subequations}
where in the last step, $\mO$ is hiding $\rho_{f}$, $\rho_{\phi}$ and constants. Note that it is more strict to assume $I^2 \eta^2 > \frac{1}{26 \rho_f^2}$ in the last step. If $I^2 \eta^2 \leq \frac{1}{26 \rho_f^2}$, the rate of upper bound becomes 
\begin{equation}
\label{aooendix: extreme_case}
\begin{aligned}
    \mO(\frac{1}{\eta T}) + 
    \mO(\frac{1}{N} \sigma_{\mG}^2) + \mO(\sigma_{H}^2) + 
    \mO(\frac{I^2 \eta^2 \kappa^2 \lambda_{\max} (\bB_N \bL^{\dagger})}{N}) .
\end{aligned}
\end{equation}
However, $I$ and $\eta$ are parameters which can be tuned/chosen while $\rho_f$ is a smoothness constant. As a consequence, $I$ and $\eta$ are our interest and we keep the vanilla form by considering $I^2 \eta^2$ is large enough, making \eqref{aooendix: extreme_case} is the special case in our result.
\end{proof}

\clearpage

\section{Proofs of Technical Lemmas}\label{appendix: lemma}

\subsection{Proof of Lemma \ref{appendix: mixed_variation}}

\begin{proof}
\label{appendix: proof_mixed_variation}
\begin{subequations}
\begin{align}
    & \sum_{k=1}^N \norm{ \nabla F_k(\bW_k) - \frac{1}{N} \sum_{j=1}^N \nabla F_j(\bW_j) }^2 \\
    = 
    & \sum_{k=1}^N \norm{
    \nabla F_k(\bW_k) - \nabla F_k(\Bar{\bW})
    + \nabla F_k(\Bar{\bW}) 
    - \frac{1}{N} \sum_{i=1}^N \nabla F_i(\Bar{\bW}) + \frac{1}{N} \sum_{i=1}^N \nabla F_i(\Bar{\bW})
    - \frac{1}{N} \sum_{j=1}^N \nabla F_j(\bW_j)
    }^2 \\
    \leq
    & 3 \sum_{k=1}^N 
    \Bigg(
    \norm{\nabla F_k(\bW_k) - \nabla F_k(\Bar{\bW})}^2
    + \norm{\nabla F_k(\Bar{\bW}) 
    - \frac{1}{N} \sum_{i=1}^N \nabla F_i(\Bar{\bW})}^2
    + \norm{\frac{1}{N} \sum_{j=1}^N (\nabla F_j(\Bar{\bW}) - \nabla F_j(\bW_j)) }^2 
    \Bigg) \notag  \\
    &
    \text{ (by Lemma \ref{appendix: Basic_algebra}) } \\
    \leq 
    & 3 \sum_{k=1}^N
    \Bigg(
    \rho_{f}^2 \norm{\bW_k-\Bar{\bW}}^2
    + \norm{\nabla F_k(\Bar{\bW}) 
    - \frac{1}{N} \sum_{i=1}^N \nabla F_i(\Bar{\bW})}^2
    +\frac{1}{N} \rho_{f}^2 \sum_{j=1}^N \norm{\Bar{\bW} - \bW_j}^2
    \Bigg) \notag \\
    &
    \text{ (by Lemma \ref{appendix: Basic_algebra} and Assumption \ref{assumption: smoothness}) } \\
    =
    & 
    6 \rho_{f}^2 \sum_{k=1}^N \norm{\bW_k - \Bar{\bW}}^2 +  3 \sum_{k=1}^N \norm{\nabla F_k(\Bar{\bW}) 
    - \frac{1}{N} \sum_{i=1}^N \nabla F_i(\Bar{\bW})}^2
\end{align}
\end{subequations}
\end{proof}

\subsection{Proof of Lemma \ref{appendix: variance_bias}}

\begin{proof}
\label{appendix: proof_variance_bias}
\begin{subequations}
\begin{align}
    &  \sum_{k=1}^N 
    \BE[\norm{ \nabla F_k(\bW^t_k) - \bg^t_k }^2] \\
    =
    &  \sum_{k=1}^N  
    \BE[\norm{ \nabla F_k(\bW^t_k) - \frac{1}{|\mB_k|} \sum_{s\in\mB_k} \nabla \hat{f}_k (\bW^t_k; \bxi_{k,s}) }^2] \\
    \leq 
    &  \sum_{k=1}^N  \frac{1}{|\mB_k|} \sum_{s \in \mB_k}
    \BE[\norm{ \nabla F_k(\bW^t_k) - \nabla \hat{f}_k (\bW^t_k; \bxi_{k,s}) }^2] 
    \text{ (by Lemma \ref{appendix: Basic_algebra}) } \\
    =
    &  \sum_{k=1}^N 
    \BE[\norm{ \nabla F_k(\bW^t_k) -  \nabla f_k(\bW^t_k; \bxi_{\mG}) }^2] +
    \sum_{k=1}^N  \frac{1}{|\mB_k|} \sum_{s \in \mB_k}
    \BE[\norm{ \nabla f_k(\bW^t_k; \bxi_{\mG}) - \nabla \hat{f}_k (\bW^t_k; \bxi_{k,s}) }^2] \\
    \leq 
    &  \sigma_{\mG}^2 +
    \sum_{k=1}^N  \frac{1}{|\mB_k|} \sum_{s \in \mB_k}
    \rho_{\phi}^2 \sum_{j=1}^N \BE[\norm{ \bh_j - \Hat{\bh}_{j \rightarrow k}}^2]
    \text{ (by Assumption \ref{assumption: bound_stochastic_gradient}) } \\
    \leq 
    &  \sigma_{\mG}^2 + N \rho_{\phi}^2 \sigma_H^2
    \text{ (by Assumption \ref{assumption: bound_hidden_representation_estimation}) } \\
\end{align}
\end{subequations}

\end{proof}

\subsection{Proof of Lemma \ref{appendix: gradient_graph_smoothing}}

\begin{proof}
\label{appendix: proof_gradient_graph_smoothing}
In this proof, some notations $\bx$, $\by$, $\alpha$, $\lambda$ and $\mu$ which are overloaded, are only valid for proving this Lemma. Additionally, for simplicity, denote $\nabla \bF \in \BR^{N \times d}$ as the gradient matrix when parameter is the aggregated one $\Bar{\bW}$, that is,
\begin{equation}
\begin{aligned}
\nabla \bF := 
\begin{pmatrix}
\nabla \bF_1^{\top}\\
\vdots \\
\nabla \bF_N^{\top}
\end{pmatrix} 
\end{aligned}
\end{equation}
where $\nabla \bF_k = \nabla F_k (\Bar{\bW}) \mbox{ } \forall k \in [N]$. So Assumption \ref{assumption: graph_smoothing} can be rewrited as:
\begin{equation}
\begin{aligned}
\sum_{(i,j)\in \mE_S} \norm{\nabla F_i(\Bar{\bW}) - \nabla F_j(\Bar{\bW})}^2 
= \tr(\nabla \bF^{\top} \bL \nabla \bF) \leq \kappa^2.
\end{aligned}
\end{equation}
Also notice that $\sum_{k=1}^N \norm{ \nabla F_k(\bW) - \frac{1}{N} \sum_{i=1}^N \nabla F_i(\bW)}^2$ is a quadratic form,
\begin{subequations}
\begin{align}
    &
    \sum_{k=1}^N \norm{ \nabla F_k(\bW) - \frac{1}{N} \sum_{i=1}^N \nabla F_i(\bW)
    }^2 \\
    \leq 
    & \sum_{k=1}^N \frac{1}{N}
    \norm{ \nabla \bF_k - \frac{1}{N} \nabla \bF^{\top} \bOnes
    }^2 \\
    =
    &
    \sum_{k=1}^N \frac{1}{N} 
    \nabla \bF_k^{\top} \nabla \bF_k
    +
    \sum_{k=1}^N \frac{1}{N^3}  \bOnes^{\top} \nabla\bF \nabla \bF^{\top} \bOnes
    -
    2 \sum_{k=1}^N \frac{1}{N}
    \nabla \bF_k^{\top} \nabla \bF^{\top} \bOnes \\
    = 
    & 
    \frac{1}{N} 
    \tr(\nabla \bF \nabla \bF^{\top})
    +
    \frac{1}{N^2} \bOnes^{\top}
    \nabla \bF \nabla \bF^{\top} \bOnes
    -
    \frac{2}{N^2} \bOnes^{\top}
    \nabla \bF \nabla \bF^{\top} \bOnes \\
    = 
    & 
    \frac{1}{N} 
    \tr(\nabla \bF \nabla \bF^{\top})
    -
    \frac{1}{N^2} \tr( \bOnes \bOnes^{\top}
    \nabla \bF \nabla \bF^{\top})\\
    =
    &
    \tr(\nabla \bF^{\top} \bB_N \nabla \bF)
\end{align}
\end{subequations}
Our objective is providing upper bound for $\tr(\nabla \bF^{\top} \bB_N \nabla \bF)$, given $\tr(\nabla \bF^{\top} \bL \nabla \bF) \leq \kappa^2$, that is,
\begin{equation}
\begin{aligned}
\label{appendix: proof_gradient_graph_smoothing_eq_origin}
\max \quad & \tr(\nabla \bF^{\top} \bB_N \nabla \bF)\\
\textrm{s.t.} \quad & \tr(\nabla \bF^{\top} \bL \nabla \bF) \leq \kappa^2\\
\end{aligned}
\end{equation}
To solve this maximization problem, list the following useful facts: \\
(i) The Laplacian matrix $\bL$ has eigenpair $(0, \bOnes)$ and only one zero eigenvalue.\\
(ii) $\bOnes^{\top} \bB_N \bOnes = 0$. \\
(ii) $\bB_N \bOnes = \bZeros$, that is, $(0, \bOnes)$ is an eigenpair of $\bB_N$.\\
Now focus on the following optimization problem,
\begin{equation}
\begin{aligned}
\label{appendix: proof_gradient_graph_smoothing_eq_opt}
\max_{\bx} \quad & \bx^{\top} \bB_N \bx\\
\textrm{s.t.} \quad & \bx^{\top} \bL \bx \leq \delta^2 \\
\end{aligned}
\end{equation}
By Karush–Kuhn–Tucker conditions, suppose $\mu$ is the Lagragian multiplier,
\begin{equation}
\begin{aligned}
\begin{cases}
2 \bB_N \bx - 2 \mu \bL \bx = 0 \\
\bx^{\top} \bL \bx \leq \delta^2 \\
\mu(\bx^{\top} \bL \bx - \delta^2) = 0
\end{cases}
\end{aligned}
\end{equation}
Focus on the case that $\mu>0$ since $\mu=0$ indicates $\bx = \bZeros$ and the objective is $0$ which is trivial. Let $0=\lambda_1 < \lambda_2 \leq \lambda_3 \leq \dots \leq \lambda_N $ be the eigenvalues of $\bL$ and $\bv_i$ is the corresponding eigenvector of $\lambda_i$ for $i \in [N]$. Note that $\bv_1 = \bOnes$. Then,
\begin{subequations}
\begin{align}
\bL &= \sum_{i=2}^N \lambda_i \bv_i \bv_i^{\top} \\
\bL^{1/2} &= \sum_{i=2}^N \sqrt{\lambda_i} \bv_i \bv_i^{\top} \\
(\bL^{1/2})^{\dagger} &= \sum_{i=2}^N \frac{1}{\sqrt{\lambda_i}} \bv_i \bv_i^{\top} \\
\Rightarrow
\bL^{1/2} (\bL^{1/2})^{\dagger} 
& = (\bL^{1/2})^{\dagger} \bL^{1/2} = \sum_{i=2} \bv_i \bv_i^{\top} 
\end{align}
\end{subequations}
Also, notice that $span \{ \bv_1,...,\bv_N\} \in \BR^{N}$, that is, $\exists \alpha_1,...,\alpha_N \in \BR$ such that $\bx = \sum_{i=1}^N \alpha_i \bv_i$, then $(\bL^{1/2})^{\dagger} \bL^{1/2} \bx = \sum_{i=2}^N \alpha_i \bv_i$, therefore,
\begin{subequations}
\begin{align}
    \bx^{\top} \bL^{1/2} (\bL^{1/2})^{\dagger} \bB_N 
    (\bL^{1/2})^{\dagger} \bL^{1/2} \bx 
    &=
    (\sum_{i=2}^N \alpha_i \bv_i)^{\top} \bB_N (\sum_{i=2}^N \alpha_i \bv_i)
    \\
    &=
    \sum_{i=2}^N \sum_{j=2}^N \alpha_i \alpha_j \bv_i^{\top} \bB_N \bv_j\\
    &=
    \sum_{i=1}^N \sum_{j=1}^N \alpha_i \alpha_j \bv_i^{\top} \bB_N \bv_j 
    \text{( by (ii) and (iii))} \\
    &=
    (\sum_{i=1}^N \alpha_i \bv_i)^{\top} \bB_N (\sum_{j=1}^N \alpha_j \bv_j) \\
    &=
    \bx^{\top} \bB_N \bx
\end{align}
\end{subequations}
As a result, denote $\by = \frac{1}{\delta} \bL^{1/2} \bx$, $\bx^{\top} \bB_N \bx = \delta^2 \by^{\top} (\bL^{1/2})^{\dagger} \bB_N (\bL^{1/2})^{\dagger} \by$, thus optimization problem \eqref{appendix: proof_gradient_graph_smoothing_eq_opt} can be expressed as,
\begin{equation}
\begin{aligned}
\max_{\by} \quad & \delta^2 \by^{\top} (\bL^{1/2})^{\dagger} \bB_N (\bL^{1/2})^{\dagger} \by\\
\textrm{s.t.} \quad & \by^{\top} \by = 1 \\
\end{aligned}
\end{equation}
This is a eigenvalue maximization problem, which shows that
\begin{equation}
\begin{aligned}
\max_{\bx^{\top} \bL \bx = \delta^2} \bx^{\top} \bB_N \bx = \delta^2 \lambda_{\max}((\bL^{1/2})^{\dagger} \bB_N (\bL^{1/2})^{\dagger}) = \delta^2 \lambda_{max} (\bB_N \bL^{\dagger})
\end{aligned}
\end{equation}
Note that $\bB$ and $\bL$ are positive semi-definite so $\delta \lambda_{\max}(\bB_N \bL^{\dagger}) \geq 0$ which indicates,
\begin{equation}
\begin{aligned}
    \max_{\bx: \bx^{\top} \bL \bx \leq \delta^2} \bx^{\top} \bB_N \bx = \delta^2 \lambda_{max}(\bB_N \bL^{\dagger})
\end{aligned}
\end{equation}
On the other hand, the original optimization problem \eqref{appendix: proof_gradient_graph_smoothing_eq_origin} can be expressed as,
\begin{equation}
\begin{aligned}
\max_{\bx_1,..,\bx_N} \quad & \sum_{i=1}^{d} \bx_i^{\top} \bB_N \bx_i\\
\textrm{s.t.} \quad 
& \bx_i^{\top} \bL \bx_i \leq \delta^2_i \mbox{, } i \in [N] \\
& \sum_{i=1}^N \delta^2_i = \kappa^2 \\
\end{aligned}
\end{equation}
It can be shown that $\kappa^2 \lambda_{\max} (\bB_N \bL^{\dagger})$ is the maximum of the above problem. As a consequence,
\begin{equation}
\begin{aligned}
    \sum_{k=1}^N \norm{ \nabla F_k(\bW) - \frac{1}{N} \sum_{i=1}^N \nabla F_i(\bW)
    }^2 
    \leq
    \tr(\nabla \bF^{\top} \bB_N \nabla \bF) 
    \leq
    \kappa^2 \lambda_{\max} (\bB_N \bL^{\dagger})
\end{aligned}
\end{equation}
\end{proof}

\subsection{Proof of Lemma \ref{appendix: between_variation}}

\begin{proof}
\label{appendix: proof_between_variation}
Notice that when $t$ is the multiple of the number of local updates $I$, $\Bar{\bW}^t - \bW^t_k = 0$ with probability $1$. So let $\mM_{I, T}:=\{t \mod I \neq 0 \text{ and } t \leq T-1\}$, then we only need to provide upper bound of the following term,
\begin{equation}
\begin{aligned}
     \sum_{t \in \mM_{I, T}} \sum_{k=1}^N \BE[\norm{\Bar{\bW}^t - \bW^t_k}^2] 
\end{aligned}
\end{equation}
Denote $t_0 < t$ as the largest multiple of $I$ when $t$ is fixed. Then $\forall \tau \in \{t_0 + 1, t_0 + 2, ..., t \}$ and $\forall k \in [N]$,
\begin{subequations}
\begin{align}
    \bW^{\tau}_k - \bW^{\tau-1}_k = -\eta \bg_{\tau - 1}^k 
    & \Rightarrow
    \bW^t_k = \bW^{t_0}_k - \eta \sum_{\tau=t_0}^{t-1} \bg^{\tau}_k \\
    \Bar{\bW}^{\tau} - \Bar{\bW}^{\tau-1} = -\eta \frac{1}{N} \sum_{k=1}^N \bg^{\tau - 1}_k
    & \Rightarrow
    \Bar{\bW}^t - \Bar{\bW}^{t_0} - \eta \frac{1}{N} \sum_{\tau=t_0}^{t-1} \sum_{k=1}^N \bg^{\tau}_k
\end{align}
\end{subequations}
Also note that $\bW^{t_0}:= \bW^{t_0}_k = \Bar{\bW}^{t_0}$, then
\begin{subequations}
\begin{align}
    &
    \frac{1}{N}\sum_{k=1}^N \BE[\norm{\Bar{\bW}^t - \bW^t_k}^2] \\
    =&
    \frac{1}{N}\sum_{k=1}^N \BE[\norm{
    \eta \sum_{\tau=t_0}^{t-1} \bg^{\tau}_k -
    \eta \frac{1}{N} \sum_{\tau = t_0}^{t-1} \sum_{j=1}^N \bg^{\tau}_j
    }^2] \\
    =&
    \frac{1}{N}\sum_{k=1}^N \BE[\norm{
    \eta \sum_{\tau=t_0}^{t-1} (\bg^{\tau}_k -
    \frac{1}{N} \sum_{j=1}^N \bg^{\tau}_j )
    }^2] \\
    =&
    \frac{1}{N} \eta^2 \sum_{k=1}^N 
    \BE[\norm{
    \sum_{\tau=t_0}^{t-1}(
    \bg^{\tau}_k -
    \nabla F_k(\bW^{\tau}_k) +
    \nabla F_k(\bW^{\tau}_k) -
    \frac{1}{N} \sum_{j=1}^N \bg^{\tau}_j +
    \frac{1}{N} \sum_{j=1}^N \nabla F_j(\bW^{\tau}_j) -
    \frac{1}{N} \sum_{j=1}^N \nabla F_j(\bW^{\tau}_j)
    )
    }^2 ]\\
    \leq &
    2\eta^2 \frac{1}{N} \sum_{k=1}^N \BE[ \norm{
    \sum_{\tau=t_0}^{t-1} 
    \Bigg( \bg^{\tau}_k -
    \nabla F_k(\bW^{\tau}_k)-
    \frac{1}{N} \sum_{j=1}^N (\bg^{\tau}_j - \nabla F_j(\bW^{\tau}_j)) \Bigg)
    }^2] \notag \\
    & +
    2\eta^2 \frac{1}{N} \sum_{k=1}^N \BE[ 
    \norm{
    \sum_{\tau=t_0}^{t-1} 
    \Bigg( 
    \nabla F_k(\bW^{\tau}_k) -
    \frac{1}{N} \sum_{j=1}^N \nabla F_j(\bW^{\tau}_j
    \Bigg)
    }^2
    ] 
    \text{ (by Lemma \ref{appendix: Basic_algebra})} \\
    = &
    2\eta^2 S_1 + 2\eta^2 S_2
\end{align}
\end{subequations}
Where 
\begin{subequations}
\begin{align}
    & 
    S_1 :=
    \frac{1}{N} \sum_{k=1}^N \BE[ \norm{
    \sum_{\tau=t_0}^{t-1} 
    \Bigg( \bg^{\tau}_k -
    \nabla F_k(\bW^{\tau}_k)-
    \frac{1}{N} \sum_{j=1}^N (\bg^{\tau}_j - \nabla F_j(\bW^{\tau}_j)) \Bigg)
    }^2] \\
    & 
    S_2:= \frac{1}{N} \sum_{k=1}^N \BE[ 
    \norm{
    \sum_{\tau=t_0}^{t-1} 
    \Bigg( 
    \nabla F_k(\bW^{\tau}_k) -
    \frac{1}{N} \sum_{j=1}^N \nabla F_j(\bW^{\tau}_j)
    \Bigg)
    }^2
    ]
\end{align}
\end{subequations}
\textbf{Bounding $S_1$:}
\begin{subequations}
\begin{align}
    S_1 :=
    & 
    \frac{1}{N} \sum_{k=1}^N \BE[ \norm{
    \sum_{\tau=t_0}^{t-1} 
    \Bigg( \bg^{\tau}_k -
    \nabla F_k(\bW^{\tau}_k)-
    \frac{1}{N} \sum_{j=1}^N (\bg^{\tau}_j - \nabla F_j(\bW^{\tau}_j)) \Bigg)
    }^2] \\
    \leq 
    &
    \frac{1}{N} \sum_{k=1}^N \BE[ \norm{
    \sum_{\tau=t_0}^{t-1} 
    (\bg^{\tau}_k -
    \nabla F_k(\bW^{\tau}_k)) 
    }^2] 
    \text{ (by Lemma \ref{appendix: Basic_algebra})} \\
    \leq
    &
    \frac{1}{N} \sum_{k=1}^N 
    (t - t_0) \sum_{\tau = t_0}^{t-1} 
    \BE[ \norm{
    \bg^{\tau}_k -
    \nabla F_k(\bW^{\tau}_k)
    }^2]
    \\
    \leq 
    & 
    \frac{1}{N}I^2 \sigma_{\mG}^2 + I^2 \rho_{\phi}^2 \sigma_H^2
    \text{ (by Lemma \ref{appendix: variance_bias})} \\
\end{align}
\end{subequations}

\textbf{Bounding $S_2$:}
\begin{subequations}
\begin{align}
    S_2 :=
    & 
    \frac{1}{N} \sum_{k=1}^N \BE[ 
    \norm{
    \sum_{\tau=t_0}^{t-1} 
    \Bigg( 
    \nabla F_k(\bW^{\tau}_k) -
    \frac{1}{N} \sum_{j=1}^N \nabla F_j(\bW^{\tau}_j)
    \Bigg)
    }^2
    ] \\
    \leq
    & 
    \frac{1}{N} (t-t_0) \sum_{k=1}^N \sum_{\tau=t_0}^{t-1} 
    \BE[ \norm{
    \nabla F_k(\bW^{\tau}_k) -
    \frac{1}{N} \sum_{j=1}^N \nabla F_j(\bW^{\tau}_j)
    }^2
    ] \\
    \leq
    &
    \frac{6}{N}I \rho_{f}^2 
    \sum_{\tau=t_0}^{t-1} \sum_{k=1}^N 
    \BE[\norm{\bW^{\tau}_k - \Bar{\bW}^{\tau}}^2]  +  
    \frac{3}{N} I
    \sum_{\tau=t_0}^{t-1} \sum_{k=1}^N
    \BE[
    \norm{ \nabla F_k(\Bar{\bW}^{\tau}) - \frac{1}{N} \sum_{j=1}^N \nabla F_j(\Bar{\bW}^{\tau})
    }^2
    ]
    \text{ (by Lemma \ref{appendix: mixed_variation})} \\
    \leq 
    &
    \frac{6}{N} I \rho_{f}^2 
    \sum_{\tau=t_0}^{t-1} \sum_{k=1}^N 
    \BE[\norm{\bW^{\tau}_k - \Bar{\bW}^{\tau}}^2]  +  
    \frac{3}{N} I^2 \kappa^2 \lambda_{\max} (\bB \bL^{\dagger})
    \text{ (by Lemma \ref{appendix: gradient_graph_smoothing})} \\
\end{align}
\end{subequations}
Thus, by result $\frac{1}{N}\sum_{k=1}^N \BE[\norm{\Bar{\bW}^t - \bW^t_k}^2] \leq 2\eta^2 S_1 + 2\eta^2 S_2$,
\begin{equation}
\begin{aligned}
     \frac{1}{N} \sum_{t=0}^{T-1} \sum_{k=1}^N \BE[\norm{\Bar{\bW}^t - \bW^t_k}^2] 
     \leq 
     & 
     \frac{2}{N} T \eta^2 I^2 \sigma_{\mG}^2 + 2 T \eta^2 I^2 \rho_{\phi}^2 \sigma_H^2
     + \\
     &
     \frac{12}{N} I \eta^2 \rho_{f}^2
     \sum_{t=0}^{T-1} \sum_{k=1}^N
     \BE[\norm{\bW^{\tau}_k -
     \Bar{\bW}^{\tau}}^2]  +
     \frac{6}{N} T I^2 \eta^2 \kappa^2 \lambda_{\max} (\bB \bL^{\dagger})
\end{aligned}
\end{equation}

Therefore, since $I \eta^2 \leq 1/13 \rho_{f}^2$ indicates $0 < \dfrac{1}{1 - 12 I \eta^2 \rho_{f}^2} \leq 13$,
\begin{subequations}
\begin{align}
     \frac{1}{N} \sum_{t=0}^{T-1} \sum_{k=1}^N \BE[\norm{\Bar{\bW}^t - \bW^t_k}^2] 
     & \leq 
     \dfrac
     {\frac{2}{N} T \eta^2 I^2 \sigma_{\mG}^2 + 2 T \eta^2 I^2 \rho_{\phi}^2 \sigma_H^2 
     +\frac{6}{N} T I^2 \eta^2 \kappa^2 \lambda_{\max} (\bB \bL^{\dagger})}
     {1 - 12 I \eta^2 \rho_{f}^2}\\
     & \leq 
     \frac{26}{N} T I^2 \eta^2  \sigma_{\mG}^2 + 26 T I^2 \eta^2 \rho_{\phi}^2 \sigma_H^2
     +
     \frac{65}{N} T I^2 \eta^2 \kappa^2 \lambda_{\max} (\bB \bL^{\dagger})
\end{align}
\end{subequations}
\end{proof}

\clearpage

\section{Additional Works}\label{appendix: works}

\subsection{Analysis for Gradient Compensation}
\label{Appendix: analysis_gradient_compensation}


To solve classification tasks formulated as a GFL problem with hidden representation sharing, as discussed in Section \ref{Sec: gradient_estimation}, we suggest a straightforward gradient estimation strategy, namely gradient compensation. To present the proposed gradient estimation scheme, further define
\begin{subequations}
\label{hidden_representation_by_time}
\begin{align}
    \bh^t_{j \rightarrow k} & = \Psi_{h} (\bx_{j}; \bW^t_{h,k}) \\
    \nabla f_k(\bW^t_k; \bXi_{\mG}) &= \phi_k(\bh^t_{1 \rightarrow k},...,\bh^t_{N \rightarrow k})
\end{align}
\end{subequations}
where $\bW^t_{h,k}$ is part of $\bW^t_k$ for client $k$ at round $t$ and $\bh^t_{j \rightarrow k}$ is the hidden representation of client $j$ for the local updates at client $k$ at round $t$. Our proposal suggests using the latest aggregated model to compute hidden representations at each communication round. Formally, suppose $t_0 < t$ is the largest multiple of $I$,
\begin{equation}
\begin{aligned}
\Hat{\bh}^t_{j \rightarrow k}
=
\begin{cases}
\Psi_{h} (\bx_{k}; \bW^t_{h,k}) & j=k\\
\frac{1}{n_j} \sum_{s=1}^{n_j} \Psi_{h}(\bx_{j,s};\Bar{\bW}^{t_0}_h)
&  j \neq k
\end{cases}
\end{aligned}
\end{equation}
where $\Bar{\bW}_{h,t_0}$ is part of $\Bar{\bW}_{t_0}$ and $n_j$ is the number of examples drawn according to the distribution of client $j$. Consequently, the biased gradient estimator employed in local updates at client $k$ has the form of
\begin{equation}
\begin{aligned}
\nabla \Hat{f}_k(\bW^t_k; \bxi_{k}) &= \phi_k(\Hat{\bh}^t_{1 \rightarrow k},...,\Hat{\bh}^t_{N \rightarrow k}).
\end{aligned}
\end{equation}

Our compensation strategy satisfies the guarantee discussed in Assumption \ref{assumption: bound_hidden_representation_estimation} with additional assumptions. Specifically, suppose Eq.\eqref{gradient_compensation} is used, that is, latest aggregated model is broadcast to estimate the hidden representation, the squared estimation error of $\Hat{\bh}_t^{j \rightarrow k}$ is bounded under following assumptions.

\begin{assumption}
\label{assumption: hidden_representation_Lipschitz}
(Lipschitz Representation Encoder)
Hidden representation encoder $\Psi_h$ defined in \eqref{hidden_representations} is Lipschitz continous with constant $\rho_h$. Formally, $\forall \bW_h, \bW_h^{\prime}$ and $\forall \bx$, $\exists \rho_{\bx}>0$ and $\rho_h := \underset{\bx}{\max} \mbox{ } \rho_{\bx}$ such that 
\begin{equation}
\begin{aligned}
    \norm{\Psi_h(\bx;\bW_h) - \Psi_h(\bx;\bW_h^{\prime})} 
    \leq
    \rho_{\bx} \norm{\bW_h - \bW_h^{\prime}}.
\end{aligned}
\end{equation}
\end{assumption}

\begin{assumption}
\label{assumption: bound_gradient}
(Bounded Gradient)
Gradient with form $\phi_k$ has bounded norm. Specifically, for any $k \in [N]$, and $\bh_1,...,\bh_N$, $\exists \Delta_k>0$ and $\Delta := \sum_{k=1}^N \Delta_k$ such that
\begin{equation}
\begin{aligned}
\norm{\phi_k(\bh_1,...,\bh_h)}^2 \leq \Delta_k.
\end{aligned}
\end{equation}
\end{assumption}
Assumption \ref{assumption: hidden_representation_Lipschitz} indicates an uniformly continuity of representation model $\Psi_h$. Assumption \ref{assumption: bound_gradient} guarantees that unbiased stochastic gradient $\nabla f_k (\bW^t_k; \bXi_{\mG})$ or biased stochastic gradient $\nabla \Hat{f}_k (\bW^t_k; \bxi_{k})$ is bounded, which ensures $\bg_t^k$ is bounded and local updates will not make an irreparable deviation. Following Lemma \ref{gradient_compensation_guarantee} formally shows that our gradient compensation is valid for the nonconvex result in Section \ref{Theory},
\begin{lemma}
\label{gradient_compensation_guarantee}
Consider the federated learning procedure described in Section \ref{Problem} under Assumptions \ref{assumption: hidden_representation_Lipschitz} and \ref{assumption: bound_gradient}. Gradient compensation \eqref{gradient_compensation} provides estimation with bounded mean squared error. Formally, for any round $t$, $\forall j \in [N]$ and $\forall k \in [N]$,
\begin{equation}
\begin{aligned}
\BE[\norm{\Hat{\bh}^t_{j \rightarrow k} - \bh^t_{j \rightarrow k}}^2 ] 
\leq 2 \eta^2 \rho_h^2 I^2 \Delta_k
\end{aligned}
\end{equation}
\end{lemma}
\begin{proof}
See Appendix \ref{appendix : proof_gradient_compensation_guarantee}.
\end{proof}
Lemma \ref{gradient_compensation_guarantee} shows that our compensation strategy \eqref{gradient_compensation} satisfies that $\sigma_k^2=2 \eta^2 \rho_h^2 I^2 \Delta_k$ and $\sigma_H^2=2 \eta^2 \rho_h^2 I^2 \Delta$ in Assumption \ref{assumption: bound_hidden_representation_estimation}. This theoretical guarantee ensures that proposed gradient compensation strategy matches nonconvex results provided in Section \ref{Theory}. 




\subsection{Proof of Lemma \ref{gradient_compensation_guarantee}}
\label{appendix : proof_gradient_compensation_guarantee}
\begin{proof}
Recall the the definitions of hidden representation and its estimator, for $j \neq k$,
\begin{subequations}
\begin{align}
    \bh^t_{j \rightarrow k} & = \Psi_{h} (\bx_{j}; \bW_{h,t}^k) \\
    \Hat{\bh}^t_{j \rightarrow k} &= \frac{1}{n_j} \sum_{s=1}^{n_j} \Psi_{h}(\bx_{j,s};\Bar{\bW}^{t_0}_h)
\end{align}
\end{subequations}
Note that $\bx_j$ and $\{\bx_{j,s}\}_{s=1}^{n_j}$ are i.i.d. Suppose $t_0<t$ is the largest multiple of $I$, we have,
\begin{subequations}
\begin{align}
    & \BE[\norm{\Hat{\bh}^t_{j \rightarrow k} - \bh^t_{j \rightarrow k}}^2 ] \\
    = & 
    \BE[\norm{\frac{1}{n_j} \sum_{s=1}^{n_j} \Psi_{h}(\bx_{j,s};\Bar{\bW}^{t_0}_h) -  \Psi_{h} (\bx_{j}; \bW^{t}_{h,k})}^2]\\
    \leq &
    \frac{1}{n_j} \sum_{s=1}^{n_j} \BE[\norm{\Psi_{h}(\bx_{j,s};\Bar{\bW}^{t_0}_h) -  \Psi_{h} (\bx_{j}; \bW^{t}_{h,k})}^2]
    \text{ (by Lemma \ref{appendix: Basic_algebra})} \\
    = &
    \frac{1}{n_j} \sum_{s=1}^{n_j} \BE[\norm{\Psi_{h}(\bx_{j,s};\Bar{\bW}^{t_0}_h) - 
    \Psi_{h} (\bx_{j}; \Bar{\bW}^{t_0}_h) + 
    \Psi_{h} (\bx_{j}; \Bar{\bW}^{t_0}_h) - \Psi_{h} (\bx_{j}; \bW^{t}_{h,k})}^2] \\
    \leq & 
    \frac{1}{n_j} \sum_{s=1}^{n_j} \bigg( 2\BE[\norm{\Psi_{h}(\bx_{j,s};\Bar{\bW}^{t_0}_h) - 
    \Psi_{h} (\bx_{j}; \Bar{\bW}^{t_0}_h)}^2]
    +
    2\BE[\norm{ \Psi_{h} (\bx_{j}; \Bar{\bW}^{t_0}_h) - \Psi_{h} (\bx_{j}; \bW^{t}_{h,k})}^2]
    \bigg)
    \text{ (by Lemma \ref{appendix: Basic_algebra})} \\
    = & 
    \frac{2}{n_j} \sum_{s=1}^{n_j}
    \BE[\norm{ \Psi_{h} (\bx_{j}; \Bar{\bW}^{t_0}_h) - \Psi_{h} (\bx_{j}; \bW^{t}_{h,k})}^2] \\
    \leq &
    2 \rho_h^2 \BE [\norm{\Bar{\bW}^{t_0}_h - \bW^t_{h,k}}^2] 
    \text{ (by Assumption \ref{assumption: hidden_representation_Lipschitz})} \\
    = &
    2 \eta^2 \rho_h^2 \BE [\norm{\sum_{\tau =t_0}^{t-1} \bg^{\tau}_k }^2] \\
    \leq &
    2 \eta^2 \rho_h^2 I \sum_{\tau =t_0}^{t-1} \frac{1}{|\mB_k|} \sum_{s \in \mB_k} 
    \BE [\norm{\phi_k(\Hat{\bh}_{1\rightarrow k}^{\tau},...,\Hat{\bh}_{N\rightarrow k}^{\tau})}^2]
    \text{ (by Lemma \ref{appendix: Basic_algebra})} \\
    \leq &
    2 \eta^2 \rho_h^2 I^2 \Delta_k
    \text{ (by Assumption \ref{assumption: bound_gradient})}
\end{align}
\end{subequations}
\end{proof}

\subsection{Mechanism of Proposed Method}
\label{Appendix: mechanism_GFL}

\textbf{Role of Hidden Encoder.}
The hidden encoder is a global model shared across the clients, which is why FL is valid (collaboration is valid). It is a model ignoring the graph information. Intuitively, $\Psi_{h}$ not only serves as a "privacy protection encoder" but also a "representation extracting model" for the target machine learning task.
With the "privacy protection encoder", raw data of any local client is not shared with both the central server and other clients. In this sense, it serves as a privacy-preserving encoder.
In addition, since clients have the same task, we assume $\Psi_{h}: \bx \rightarrow \bh$ is shared to extract representations for the target task. 
Note that if we set $\Psi_{\mG}$ as the identity mapping (ignore the graph information), our solution reduces to the conventional FL solution to learn a global model $\Psi_{h}$. 

\textbf{Role of Neighborhood Aggregator.} The neighborhood aggregator in our solution is the personalized model, which uses graph information to account for statistical heterogeneity. We use the term "personalization" because graph representation for each client is personalized by the graph-based model $\Psi_{\mG}$. For example, the most simple neighborhood aggregator is the simple graph convolution: $\Psi_{\mG}(\bH) = \bA \bH$ (without learnable parameters), which uses $\bA$ as the smoothing matrix.
By involving the neighborhood aggregator, our method can be viewed as a personalized FL. 
Personalization (or graph) is necessary due to the existence of statistical heterogeneity in multi-client systems (see Figure \ref{fig: graph and connectivity}). 
In short, $\Psi_{\mG}$ serves as modeling the heterogeneity using graph $\mG$. Note that when $\mG$ does not fully capture the Non-IID relationship across clients, $\bW_{\mG}$ serves as weights for adjusting the neighborhood aggregation level using $\mG$ to approximate the true structure of heterogeneity.

\subsection{Privacy Protection in GFL}
\label{Appendix: privacy_GFL}

FL requires the protection of node-level privacy: client can not share their own collected data with both other clients and the central server directly. This node-level privacy is supposed to be protected in GFL. According to \citep{zhu2019deep}, sharing both model parameters, hidden representations, and gradients to the clients is subject to the raw data leakage problem, which violates the node-level privacy. In other words, directly sharing $\{\bh_j, \nabla \bh_j\}_{j=1}^N$ raises the concern about raw data recovery by untrustworthy clients or the central server. Fortunately, our proposed solution in $\S$ \ref{Problem}, which allows sharing of hidden representations and the corresponding gradients during the communication between clients and the central server, does not violate node-level privacy. We provide a detailed explanation by considering the following two cases: \textbf{(I) Raw data recovered by clients}. \textbf{(II) Raw data recovered by the central server.}

(I) Our hidden representation sharing technique can strictly avoid the first case due to the utility of personalized neighbor aggregator $\Psi_{\mathcal{G}}$. 
Thanks to graph $\mG$, the central server allocates the uploaded hidden representations and their gradients and broadcasts the graph-based aggregated representations and gradients to the clients, indicating that each client will not directly receive the latest hidden representation and the corresponding gradient of its neighborhood. Instead, each client receives aggregated information about hidden representations and the gradients. For instance, in $\S$ \ref{Scenario}, in proposed \texttt{GFL-APPNP} with notations given in Algorithm \ref{alg: GFL-APPNP}, client $k$ can only receive the latest $C_k$ and $\nabla C_k$ for each communication round. Note that even though a client can know who its neighbors are, it has no access to the propagation matrix $\tilde{A}$, since the number of steps for propagation is only known by the central server. As a result, client $k$ has no chance to recover the raw data of its neighborhood. 

(II) The second case is also not a concern if one assumes the central server is trustworthy. If the central server is perfectly protected (central server is not possible to be an attacker or be hacked), since raw data is not shared with the central server, node-level privacy is protected. In the extreme case that the central server is not reliable, we provide the following strategies to address this concern. First, improve the design of the central server. For example, use one server for receiving and allocating hidden representations and use the other server for gradients. Another way is releasing memories of hidden representations after broadcasting aggregated results and receiving the gradients after the release. Second, according to \citep{zhu2019deep}, one possible defense strategy is adding random noises to the gradients (noisy gradients), and we provide the corresponding experimental results using 
DP method in Appendix \ref{appendix: noisy_gradient}.

Moreover, considering the classification tasks on graphs, described in Section \ref{sec: clissification_tasks_graph}, only the deterministic node classification is supposed to be discussed in the above cases. Stochastic node classification and supervised classification are impossible for clients and central server to receive $\{\bh_j, \nabla \bh_j\}_{j=1}^N$ directly. The reason is that clients have multiple feature vectors and only upload the average hidden representations and gradients (See Eq.\eqref{gradient_compensation}).

Finally, GFL is protecting data privacy by succeeding the nature of FL. Therefore, unlike differential privacy (DP), our GFL is not designed for measuring the quantity of privacy level and corresponding performance loss based on privacy. Even though it is interesting to combine the DP method in GFL, we leave the comprehensive work on DP in GFL as a potentially promising future work. We only experiment noisy gradient for GFL in Appendix \ref{appendix: noisy_gradient}, inspired by DP methods.

\clearpage

\section{Supplement to Experiments}\label{appendix: experiments}

\subsection{Tasks Description}
\label{Appendix: experiment_task}

\textbf{Deterministic Node Classification.} The most common classification task on relational data is semi-supervised node classification which is transductive learning. In this work, we call the classical semi-supervised node classification the deterministic node classification since $\bxi_k$ for each node is deterministic with one feature vector and one label. In this task, a graph-based model is going to classify the label of the test nodes given the feature vectors and the graph structure for all nodes and the label of training nodes. GFL problem is the federated learning version of this task by setting $\bxi_k= (\bx_k, y_k)$ is from a deterministic distribution which means $\bxi_{\mG}$ is not random. Note that in GFL, the unlabeled clients do not participate in local updates, but they have to compute and upload hidden representations. Also, the connectivity mentioned in Assumption \ref{assumption: graph_smoothing} is on the graph of labeled clients. The local updates to solve the GFL problem in this task are biased gradient descent using the deterministic feature vector and the estimated hidden representations. Note that this is the special case of stochastic node classification when local data of node $k$ is assumed to be degenerated distributed.

\textbf{Stochastic Node Classification.} Suppose the graph structure describes the relationship between the distributions among the nodes, the extended version of the classical semi-supervised node classification task on a graph is classifying the label of the test nodes given the samples from all nodes and the graph structure as well as the fixed labels of training nodes. This task assumes the randomness of $\bxi_k$ is only from $\bx_k$ but not $\by_k$. In other words, this task can be formulated by the GFL problem in Section \ref{Problem} by assuming the randomness of $\bXi_{\mG}$ is only from $\bX$ but not $\bY$. GFL is the federated learning version of this task. Similar to the deterministic one, in GFL, the unlabeled clients do not participate in local updates while they must share hidden representations, and the connectivity for convergence refers to the graph of labeled clients. The local updates to solve the GFL problem in this task is biased SGD with batches sampled from the given examples. This task is novel to the current graph machine learning community and harder than deterministic node classification. The applications of this task are discussed as follows with two practical examples: (i) User demographic label prediction in a social network. Given the network of users for some cell phone app and suppose users' cell phones provide the computational resources and collect the behavior data, predicting the demographic label (such as gender, age group, etc.) of a new user with behavior data in the network is the stochastic node classification problem. (ii) Classifying hospitals. Based on the network of hospitals and the patients' records, classifying a new hospital in the network is also a stochastic node classification.

\textbf{Supervised Classification.} This task assumes the randomness of $\bxi_k$ is from both $\bx_k$ and $\by_k$, indicating that each node might have examples with different labels. In other words, This task assumes the randomness of $\bXi_{\mG}$ is from both $\bX$ and $\bY$. The objective of this task is to classify the feature vector in some clients, given training data among all nodes with a graph structure. Similar to stochastic node classification, GFL is the federated learning version of this task, and local updates to solve the GFL problem in this task are biased SGD with batches sampled from the given examples. One application is also based on the users' network from some cell phone app: suppose users' cell phones provide the computational resources and collect the hourly activity data. Labeling the incoming data of the users in the network is a supervised classification task. Another application is classifying patients in a hospital located as a node in a hospital network: suppose some hospital has an insufficient record for some special type of patients, information from other hospitals in the network can help classify these patients in that hospital, and this refers to the supervised classification task with LoG.

Our GFL setting introduced in Section \ref{Problem} is not only for standard supervised learning but also can be easily extended to semi-supervised client classification by following additional constructions. The labeled clients perform local updates while unlabeled ones only need to share hidden representations. The connectivity of graphs with labeled clients, which is a subgraph of $\mG$, will directly affect convergence performance as discussed in Section \ref{Theory}.

\subsection{Data Generations in Section \ref{Experiments}}
\label{Appendix: data_generation}
This part provides a detailed description of the data generation (including synthetic data and real data) for our experiments in Section \ref{Experiments}. There are three settings in our experiment: deterministic node classification using simulated data, deterministic node classification using real data, and stochastic node classification using simulated data. In our data generation procedure, cSBMs are utilized multiple times, and we will start with a brief recap of cSBMs and then the detailed data generation for three settings in our empirical study. For an illustrative purpose, we overload some notations in this Section (Appendix \ref{Appendix: data_generation}), and these notations are only valid in this part. In other words, the notations used in this Section are not the same as in other parts of our paper.

\label{Appendix: recap of cSBMs}
\textbf{Recap of cSBMs.} Contextual Stochastic Block Models (cSBMs) is proposed by \citep{deshpande2018contextual}. In this section, we briefly present how to generate a synthetic graph $G$ using cSBMs. Suppose there are $N$ nodes form two clusters of approximately equal size with node labels $\by \in \{-1,1\}^{N}$. Each node is assigned with a $p$ dimension Gaussian feature vector $\bx_i = \sqrt{\dfrac{\mu}{N}}y_{i}\bu + \dfrac{\bZ_i}{p}$. Where $\bu \sim N(0, \bI_p/p)$, $\bZ_i \in \BR^{p}$ has independent standard normal entries, and $\mu$ is a hyperparameter that is fixed and known. The adjacency matrix $\bA$ defines the undirected graph is determined by the following probabilistic model:
\begin{equation}
\begin{aligned}
\BP(A_{i,j} = 1) 
= \begin{cases}
\dfrac{d+\lambda\sqrt{d}}{N}, \text{if $y_i = y_j$} \\
\dfrac{d-\lambda\sqrt{d}}{N}, \text{otherwise}
\end{cases}
\end{aligned}
\end{equation}
                             
Where $d$ denotes the average degree, and $\lambda$ denotes the normalized degree separation, which is also fixed and known for simplicity. In cSBMs, the difference between the means of two clusters is controlled by a parameter $\mu$, and the difference between the edge densities within each cluster and the edge densities between clusters is controlled by a parameter $\lambda$. Asymptotically, one needs $\lambda^2 + \dfrac{\mu^2}{N/p} > 1$ to ensure a vanishing ratio of the misclassified nodes and the total number of nodes. As introduced in \citep{chien2020adaptive}, we use $\phi = \arctan(\dfrac{\lambda\sqrt{N/p}}{\mu}\times\dfrac{2}{\pi})$ to control the extent of information carried by the node features and graph topology. Note that $\phi = 0$ indicates that only node features are informative while $\abs{\phi} = 1$ indicates that only the graph topology is informative.

\textbf{Synthetic Data for Deterministic Node Classification.} In our experiment, $20$ synthetic graphs with one example per node are synthesized, and each of these $20$ graphs is generated by using cSBMs with $d=8, \lambda=2, \mu=1, N=200, p=100$, which are parameters described in the previous recap part. Then we fix the graph topology for all 20 generated synthetic graphs so they have the same structure with the same number of nodes, edges, and labels. We also fix the train-valid-test split ($10\%/10\%/80\%$) for all $20$ graphs. The training set is balanced, which indicates that it has the same number of nodes for both two classes. Moreover, the subgraph induced by the nodes in the training set is connected. The $\phi$ value is approximately $0.78$, and the dimension for the feature vector is $100$. To achieve randomness in repeated simulations, we change the node features for each of these $20$ synthetic graphs.

\textbf{Real Data for Deterministic Node Classification.}  Cora dataset is a citation graph that is widely used as a benchmark dataset for the semi-supervised node classification task. This dataset consists of 2708 nodes (scientific publications) classified into one of seven classes, and it consists of $5429$ edges. The dimension of a node feature vector is $1403$, which is a $0/1$-valued word vector indicating the absence/presence of the corresponding word from the dictionary, which consists of $1433$ unique words. As our discussion in Section \ref{experiment: deterministic_node_classification}, due to our limited computational resources, instead of using the whole Cora for our experiments, we extract subgraphs with $300$ nodes from Cora, namely subCora. The following steps describe the generation process of subCora. For each random seed from a large pool of random seeds:

\begin{itemize}
   \item[\textbf{1)}] Randomly select $800$ nodes from Cora using one random seed and find all connected components from this subgraph with $800$ nodes.
    \item[\textbf{2)}] For each connected component, we keep it as one training set if it satisfies the following two conditions: (i) Size must fall between $30$ and $50$ including $30$ and $50$, which can ensure our training set only counts toward $10-15\%$ of the entire subCora graph with $300$ nodes (ii) The connected component have all $7$ classes, and the standard deviation of the counts for different classes is less than or equal to $2.5$. Securing that our training set is as balanced as possible.
    \item[\textbf{3)}]  To generate the corresponding validation set and the testing set, we find all reachable nodes from the training set using BFS with a maximum depth of $2$, and select nodes from these reachable nodes as a validation set such that the validation size is equal to the training size. Then select testing nodes so that our subCora graph has $300$ nodes in total. This step makes the extracted subCora graphs have good connectivity, which matches our Assumption \ref{assumption: graph_smoothing}. 
\end{itemize}

By simulating subCora in this way, we make sure that the training sets are connected with all $7$ classes and relatively good class balance. Also, the generated subCora graphs have relatively better connectivity than randomly selecting testing sets and validation sets.

\textbf{Synthetic Data for Stochastic Node Classification.}Unlike deterministic node classification, this task requires multiple examples for each node. The vanilla cSBMs are not designed for simulating data for this task. However, since statistical heterogeneity indicates the graph homophily on feature distribution, we can still utilize cSBMs to generate synthetic relational data with one label and multiple feature vectors for each node. Then the synthetic data can be generated by these distributions among clients. First, we use the original probabilistic model from cSBMs (\ref{Appendix: recap of cSBMs}) to generate the graph structure (i.e., edges and nodes), and we assign a $\pm1$ label to each node by either manually assigning or using a Bernoulli distribution with success probability $0.5$ to achieve a balance class set. To be more specific, the hyperparameters we used to generate the graph structure are $d=10, \mu=1, \lambda=2$ and $N=200$. 

Select one training set that is connected and achieves class balance by simply selecting nodes randomly until the subgraph induced by these nodes is connected and it has class balance. Then randomly select the validation set, and testing set such that we have a $10\%/10\%/80\%$ spilled. 

Denote the label vector as $\bv = \{v_1,\dots,v_N\} \in \{\pm1\}^{N}$. And draw a $\bu \sim N(0, \bI_p/p)$ where $p=100$,  and it will be used later for feature vectors generation purpose. With these hyperparameters, the $\phi$ is set to be $0.78$. Each node has $40$ local data points. The data sampling process for each node is described in the following,

For each node $i=1$ to $i=N$,

\begin{itemize}
    
    \item[\textbf{1)}]  The $40$ labels for client $i$ is the just the $40$ repeats of $v_i$, denote this vector as $\by_i = \{v_i,\dots,v_i\} \in \{\pm1\}^{40}$.

    \item[\textbf{2)}] To generate each of it's $40$ local feature vectors denoted as $\bX_i =  \{\bx_{i,1}^{T},\dots,\bx_{i,40}^{T}\} \in \BR^{40\times p}$, we utilize the feature generation mechanism from cSBMs, such that $\bx_
    {i, j} = \sqrt{\dfrac{\mu}{N}}v_i\bu + \dfrac{\bZ_{i,j}}{p}$. And just like cSBMs, $\bZ_{i, j} \in \BR^{p}$ has independent standard normal entries.
    
\end{itemize}

We fix the hyperparameters and graph structure, then repeat the sampling procedure for $20$ times to generate $20$ synthetic graphs that only differ in node/client level feature vectors for random experiments purposes.


\textbf{Synthetic Data for Supervised Classification.} We will use the Bernoulli distribution and multivariate Gaussian distribution to generate labels and feature vectors for each client/node. First, we use the original probabilistic model from cSBMs (\ref{Appendix: recap of cSBMs}) to generate the graph structure (i.e., edges and nodes), and we assign a $\pm1$ tag (not labels, which will be generated later.) to each node by either manually assigning or using a Bernoulli distribution with success probability $0.5$ to achieve a balance tag set. To be more specific, the hyperparameters we used to generate the graph structure are $d=5, \mu=0.1, \lambda=2.2$ and $N=50$. Importantly, We make sure that the entire synthetic graph is connected, and this can be easily done by selecting a slightly higher $d$ (average degree) and repeating the graph generation process until the generated synthetic graph is connected. Denote the tag vector as $\bv = \{v1,\dots,v_N\} \in \{\pm1\}^{N}$. And draw a $\bu \sim N(0, \bI_p/p)$ where $p=100$,  and it will be used later for feature vectors generation purpose. With these hyperparameters, the $\phi$ is $0.96$. The data sampling process for each node is described in the following,

For each node $i=1$ to $i=N$,
\begin{itemize}

    \item[\textbf{1)}]  If the tag for node $i$ ($v_i$) is $-1$, assign a Bernoulli distribution with success probability $\dfrac{3}{10}$, else assign a Bernoulli distribution with success probability $\dfrac{7}{10}$. Then draw $120$ independent samples from this Bernoulli distribution. These independent samples serve as the local labels for this node $i$. Denote it as $\by_i = \{y_{i,1},\dots,y_{i,120}\}\in \{\pm1\}^{120}$.

    \item[\textbf{2)}] To generate each of it's $120$ local feature vectors denoted as $\bX_i =  \{\bx_{i,1}^{T},\dots,\bx_{i,120}^{T}\} \in \BR^{120\times p}$, we utilize the feature generation mechanism from cSBMs, such that $\bx_
    {i, j} = \sqrt{\dfrac{\mu}{N}}y_{i,j}\bu + \dfrac{\bZ_{i,j}}{p}$. And just like cSBMs, $\bZ_{i, j} \in \BR^{p}$ has independent standard normal entries.
    
\end{itemize}

We fix the hyperparameters and graph structure, then repeat the sampling procedure for $20$ times to generate $20$ synthetic graphs that only differ in clients' data for random experiments purposes. In a supervised classification setting, we utilize data from all clients/nodes and split the local data for each client/node into a train, validation, and test set. The feature vectors and labels in the validation test are used only for hyperparameters tuning, and the test set is used only once to calculate test accuracy. In our setting, the train-valid-test split is set to be $8.33\%/8.33\%/83.33\%$ (i.e., $10/10/100$). A relatively small data set with labels available (i.e., train and valid) for each client is reasonable to be assumed in federated learning problems, and a relatively small graph ($50$ nodes) is a reasonable assumption under supervised classification tasks in real-world applications.

\subsection{Model Description in Section \ref{Experiments}}
\label{Appendix: model_description}
This part provides details about models used in Section \ref{Experiments}. For training on all models, including baseline models, we use the SGD optimizer with optimized learning rates based on the specific task and models. During training, we keep tracking the lowest validation loss and the corresponding model and use this model to report test accuracy. For baseline models including \texttt{GAT}, \texttt{GCN}, and \texttt{SAGE}, we adapt the well-optimized original structure. However, in order to achieve a fair comparison with our method, we do not include layers like dropout on the adjacency matrix layer, no bias, and we train all baseline models using the basic SGD optimizer with an optimized learning rate instead of the Adam optimizer with weight decay (L2 regularization) which is widely used for training on these baseline models. For baseline models, including \texttt{FedMLP} and \texttt{MLPs} implemented for supervised learning tasks on synthetic graphs, we use a two-layer Multi-Layer-Perceptron (MLP) with $64$ hidden units and no bias term.

\textbf{Graph Federated Learning for APPNP (\texttt{GFL-APPNP}).} For all experiments we conduct, including deterministic node classification and stochastic node classification, we use a two-layer Multi-Layer-Perceptron (MLP) with $64$ hidden units and no bias term. We fix the teleport probability $\alpha$ to be $0.1$ and the total steps for personalized propagation $C$ to be $10$, following the APPNP model in \citep{klicpera2018predict} and \cite{chien2020adaptive}. We train \texttt{GFL-APPNP} for different $I \in {1, 10, 20, 50}$ with gradient compensation and without gradient compensation. Note that different $I$ will lead to different numbers of communications. For example, if we run $3000$ updates for $I=10$, then the number of communications will be $3000/10=300$. For the deterministic node classification task on both synthetic graphs and subCora graphs, after optimizing the learning rate over $\{0.01, 0.02, 0.05, 0.1, 0.5\}$, we select $0.5$ as the best learning rate for synthetic graphs and $0.02$ as the best learning rate for subCora graphs. An interesting observation is that even if $0.5$ is a relatively large learning rate, it works because the landscapes for the training loss on synthetic graphs are very smooth. We run $4000$ updates for subCora graphs and $3000$ updates for synthetic graphs. The same learning rates and the numbers of updates are adapted for both variants with gradient compensation and variants with no compensation. For the stochastic node classification task on synthetic graphs, we use a learning rate of $0.2$ and a batch size of $40$ (full batch) and run $5000$ updates for each variant with gradient compensation. For supervised classification tasks on synthetic graphs, we use a learning rate of $0.2$ and batch size of $5$ and run $2000$ updates for each variant with gradient compensation.


\textbf{Graph Convolutional Networks (\texttt{GCN}).}
The most common baseline is \texttt{GCN} originated from \citep{kipf2016semi}. In our experiment, we use two GCN layers with $64$ hidden units following the well-optimized original model structure in \citep{kipf2016semi}. For the deterministic node classification task on synthetic graphs, we train $4000$ updates with a learning rate equal to $0.1$. For the deterministic node classification task on subCora, we train $4000$ updates with a learning rate equal to $0.01$. These learning rates are determined by comparing the lowest validation losses.

\textbf{Graph Attention Networks (\texttt{GAT}).}
Another baseline model, \texttt{GAT}, is conducted in our experiment following \citep{velivckovic2017graph}. Particularly, we use $2$ GAT convolutional layers where the first layer has $8$ attention heads, and each head has $8$ hidden units, the second layer has $1$ attention head and $64$ hidden units following the well-optimized original model structure in \citep{velivckovic2017graph}. For the deterministic node classification task on synthetic graphs, we train $4000$ updates with a learning rate equal to $0.02$. For the deterministic node classification task on subCora graphs, we train $4000$ updates with a learning rate equal to $0.01$. These learning rates are determined by comparing the lowest validation losses.

\textbf{Graph Neural Networks with Sample and Aggregate (\texttt{GraphSAGE}).} One popular GNN model we treated as a baseline is \texttt{GraphSAGE} which originated from \citep{hamilton2017inductive}. In our empirical studies on deterministic node classification, we use $2$ SAGE convolutional layers with $64$ hidden units. For the node classification task on synthetic graphs, we train $5000$ updates with a learning rate equal to $0.02$. For the deterministic node classification task on subCora, we train $4000$ updates with a learning rate equal to $0.01$.

\textbf{Fedederated Learning for Multi-Layer-Perceptron (\texttt{FedMLP}).} We follow a common federated learning procedure with no graph information and data sharing involved. we train this baseline model with different $I \in \{10, 20, 50\}$. For the supervised learning task on synthetic graphs, we train $2000$ updates with a learning rate of $0.1$ and a batch size of $5$.

\textbf{Multi-Layer-Perceptrons (\texttt{MLPs}).} In this baseline model, each client will have one local \texttt{MLP} model that only utilizes its own train set for training and validation set for tuning. Each client uses the local model with the lowest validation loss calculated by its own valid set during training to report the test accuracy of its own test set. The average test accuracy across all clients is calculated and reported by the central server. For the supervised learning task on synthetic graphs, we train $200$ updates with a learning rate of $0.1$ and a batch size of $5$.

\subsection{Additional Experiments}
\label{Appendix: additional_experiments}

In this section, we provide additional experimental results, including the experiments for additional table for Section \ref{experiment: stochastic_node_classification}, our algorithm combining DP method and a summary table for all experiments conducted.

\subsubsection{Additional Experiment I: Table for Stochastic Node Classification}

We provide a summary Table \ref{table: SNC_results} for proposed \texttt{GFLAPPNP} with $I \in \{1, 10, 20, 50\}$, which is in the context of Section \ref{experiment: stochastic_node_classification}. The specific structure, hyperparameters and synthetic graphs used for this empirical study are described in Appendix \ref{Appendix: data_generation} and \ref{Appendix: model_description}.

\begin{table}[h!]
\caption{\footnotesize Results on Stochastic Node Classification. This table provides the results of average test accuracy and corresponding $95\%$ confidence interval for stochastic node classification task. The circled row represents the highest test accuracy on average.}
\label{table: SNC_results}
\centering
\begin{tabular}{c c} 
\hline
& Synthetic Graphs\\
\hline
\texttt{GFL-APPNP} $I=1$ & $\boxed{98.7 \pm 0.26\%}$ \\
\texttt{GFL-APPNP} $I=10$ & $92.4 \pm 0.19\%$ \\
\texttt{GFL-APPNP} $I=20$ & $92.5 \pm 0.17\%$ \\
\texttt{GFL-APPNP} $I=50$ & $92.5 \pm 0.17\%$ \\
\hline
\end{tabular}
\end{table}





\subsubsection{Additional Experiment II: Effect of Graph Connectivity (Figure \ref{fig: graph and connectivity})}

In this section, we provide details for Figure \ref{fig: graph and connectivity}(a). We use the same data generation process in \ref{Appendix: data_generation} for supervised classification to generate four synthetic graphs with different connectivity measured by $\lambda_{\max} (\bB_N \bL^{\dagger})$, each synthetic graph will repeat the data sampling process $20$ times to achieve randomness. We use the same model which is our method with $64$ hidden units, $I=10$, and same model initialization to conduct experiments on all four synthetic graphs and their corresponding $20$ repetitions. All four synthetic graphs has the same hyperparameters $N=40, \mu=1, \lambda=2, \text{ and } p=100$ expect for one hyperparmeter $d \in \{25, 15, 10, 5\}$, thus they all have the same $\phi = 0.574$. Recall that $d$ represents the average degree for a synthetic graph, so a higher $d$ naturally leads to a higher connectivity, and the results are four different $\lambda_{\max} (\bB_N \bL^{\dagger})$ values $\{1.76\times 10^{-3}, 4.17\times 10^{-3}, 1.09\times 10^{-2}, 7.52\times 10^{-2}\}$. All nodes have the same number of local data points $120$ for all four synthetic graphs, and the train-valid-test split is $10/10/100$. SGD optimizer is used to train $1500$ updates for our method with a learning rate $0.5$, and a batch size of $5$ for all four synthetic graphs.

\subsubsection{Additional Experiment III: Necessity of Graph Structure in GFL (Figure \ref{fig: graph and connectivity})}

As our discussion in Section \ref{Problem}, the network of clients in multi-client systems accounts for the statistical heterogeneity problem. This empirical study aims to show that the heterogeneity problem is non-negligible. Our experiment is based on the context of supervised classification with baseline models \texttt{MLPs} and \texttt{FedMLP}, matching the same setting in Section \ref{experiment: supervised_classification}. The result is given by boxplot on Figure \ref{fig: graph and connectivity}(b). The models and details are the same as Section \ref{experiment: supervised_classification}, which can be found in Appendix \ref{Appendix: data_generation} and \ref{Appendix: model_description}.

\subsubsection{Additional Experiment IV: Differential Privacy}
\label{appendix: noisy_gradient}

In this section, we conduct an empirical study on DP-based Graph Federated Learning for DNC task (semi-supervised node classification) on semi-synthetic dataset subCora (See Appendix \ref{Appendix: data_generation}. The motivation is discussed in Section \ref{Sec: gradient_estimation}. For convenience, our experiments are also called subCora DP experiments. For all subCora DP experiments, we used SGD with a learning rate of $0.01$ and $4000$ updates. We used the same model structure described in \ref{Appendix: model_description}. We implemented two strategies for adding noises to our method. The first one injects noise to both hidden representations and gradients, and the second one only applies noisy hidden representations. Following is the detailed description.

\textbf{Noisy hidden representations($\bh_k^{t}$).} Instead of uploading latest hidden representation $\bh_k^{t}$ of each client $k$ at each communication round $t$ where $t\mod I = 0$ to central server, we add a random noise vector denoted as $\bepsilon_{h, k}^{t}$ that has independent standard normal entries to the hidden representation, and upload $\bh_k^{t} + \bepsilon_{h, k}^{t}$to central server. We use different standard deviations for the noise injected in hidden representations. 
Figure \ref{fig: hn} shows that when we adapt noisy hidden representations, both train and validation loss are NOT significantly affected for all standard deviations and $I \in \{10,20\}$.

\textbf{Noisy hidden representations and gradients ($\bh_k^{t}, \nabla{\bh}_k^{t}$).} We add Gaussian noise on both hidden representations and corresponding gradients before each communication round. Formally, we inject noises $\bepsilon_{h, k}^{t}$ and $\bepsilon_{g, k}^{t}$ that has independent standard normal entries to the hidden representation and gradient vector respectively, and upload $\bh_k^{t} + \bepsilon_{h, k}^{t}$ and $\nabla{\bh}_k^{t} + \bepsilon_{g, k}^{t}$ to central server. We use different standard deviations for the noise injected in both hidden representations and corresponding gradients.  
Figure \ref{fig: hngn} shows that when we adapt noisy hidden representations and gradients. Noise with a relatively large standard deviation will affect the validation loss by some degree (Zoom in for better viewing quality). A larger standard deviation will result in a larger final validation loss. However, the training loss is not significantly affected. The effect of noise is more significant for larger $I$.

In all plots and tables in this experiment, the label "\textbf{I=10 hn0.5\_gn0.5}" denotes our method \texttt{GFL-APPNP} with $I=10$, and we add noises to both hidden representations and corresponding gradients with standard deviation equals $0.5$. The label "\textbf{I=10 hn\_0.5}" denotes that for our method with $I=10$, we add noise to hidden representations only with a standard deviation equal to $0.5$. In what follows, we will use standard deviation to denote the standard deviation of the added noises.

Moreover, we also investigated the effect of the standard deviation of the noise on the test set accuracy and the effect of $I$ when we fix the noise strategy and standard deviation of the noise. Figure \ref{fig: fixstd} shows that when we fix the noise strategy and standard deviation of the noise. A larger standard deviation for the Gaussian noise will result in worse performance (larger loss and higher variation) during training as well as validation. This matches the general situation when applying DP-based privacy-preserving strategies. Also, one can conclude that larger $I$ will result in a larger final validation loss which is more significant, while a larger standard deviation of the noise is applied. Figure \ref{fig: effectofstd} shows that in general, when we fix the noise strategy and $I$, a larger standard deviation will lead to lower average test set accuracy. As one may expect, there is a trade-off between performance and privacy protection, matching the results of DP-based methods. But even with a relatively large standard deviation, the average test accuracy is still competitive compared with all other cases. This reflects that DP-based optimization is applicable under the proposed GFL setting. The exact numbers of test accuracy can be found in Table \ref{table: dp_subcora}.

\begin{table}[h!]
\caption{\footnotesize This table provides the test set accuracy for all DP experiments on subCora task. We use SD to denote the standard deviation of the added noise. When the standard deviation is $0$, it represents our method \texttt{GFL-APPNP} with no additional noise.}
\label{table: dp_subcora}
\centering
\begin{tabular}{ c | c c c c c} 
\hline
SD & $0$ & $0.25$ & $0.5$ & $0.75$ & $1$\\
\hline
\texttt{$I=10$ hn} & --- & $54.6 \pm 3.98\%$ & $54.2 \pm 4.07\%$ & $53.9 \pm 3.93\%$ & $53.5 \pm 3.98\%$\\
\texttt{$I=20$ hn} & --- & $54.5 \pm 4.08\%$ & $54.5 \pm 4.13\%$ & $53.7 \pm 4.06\%$ & $53.7 \pm 4.13\%$\\
\texttt{$I=10$ hn\_gn} & --- & $54.4 \pm 3.92\%$ & $53.8 \pm 3.76\%$ & $53.4 \pm 3.64\%$ & $52.9 \pm 3.79\%$\\
\texttt{$I=20$ hn\_gn}& --- & $54.3 \pm 3.93\%$ & $54.5 \pm 4.13\%$ & $52.5 \pm 3.76\%$ & $52.0 \pm 3.98\%$\\
\texttt{$I=10$} & $54.6 \pm 3.94\%$ & --- & --- & --- & ---\\
\texttt{$I=20$} & $54.6 \pm 3.99\%$ & --- & --- & --- & ---\\
\hline
\end{tabular}
\end{table}

\begin{figure*}[h!]
\centering
\includegraphics[scale = 0.12]{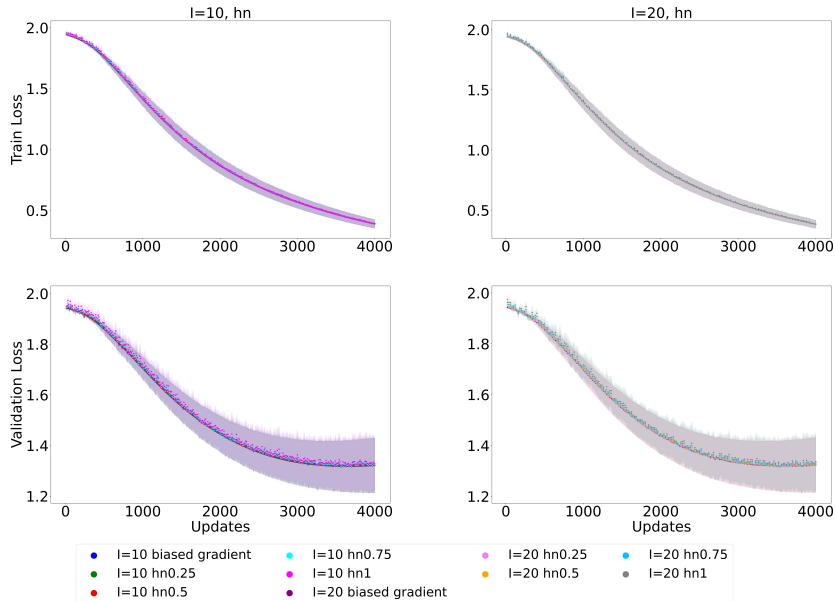}
\vspace{-5mm}
\caption{This figure shows the results for noisy hidden representations with different standard deviations of our method with $I \in \{10,20\}$. The first column contains train and validation loss for $I=10$, and the second column contains train and validation loss for $I=20$. As we can see, adding noise to hidden representations does not affect train and validation loss by a significant amount for both $I=10$ and $I=20$, even with a relatively large standard deviation.}
\label{fig: hn}
\end{figure*}

\begin{figure*}[h!]
\centering
\includegraphics[scale = 0.12]{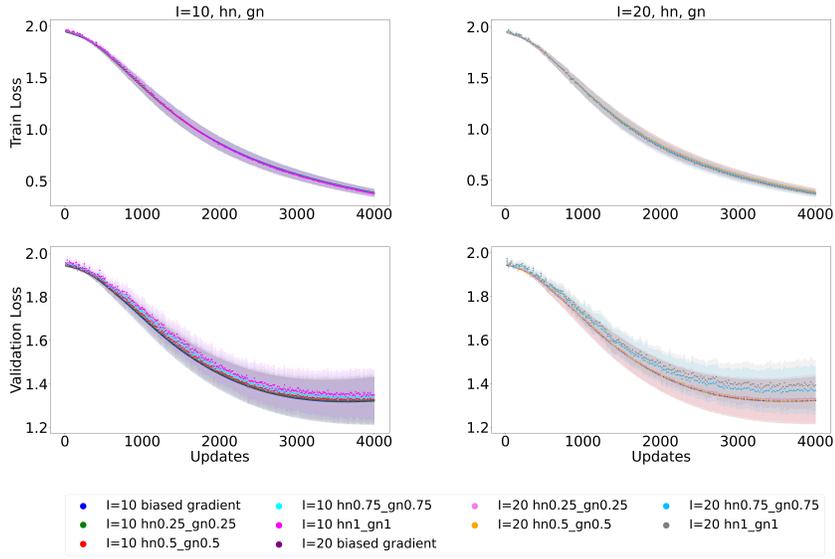}
\vspace{-8mm}
\caption{This figure shows the results for noisy hidden representations and gradients with different standard deviations of our method with $I \in \{10,20\}$. The first column contains train and validation loss for $I=10$, and the second column contains train and validation loss for $I=20$. As we can see, adding noise with a relatively large standard deviation will affect the validation loss by some degree for both $I=10$ and $I=20$. The effect of noise is more significant when $I=20$.}
\label{fig: hngn}
\end{figure*}

\begin{figure*}[h!]
\centering
\includegraphics[scale = 0.12]{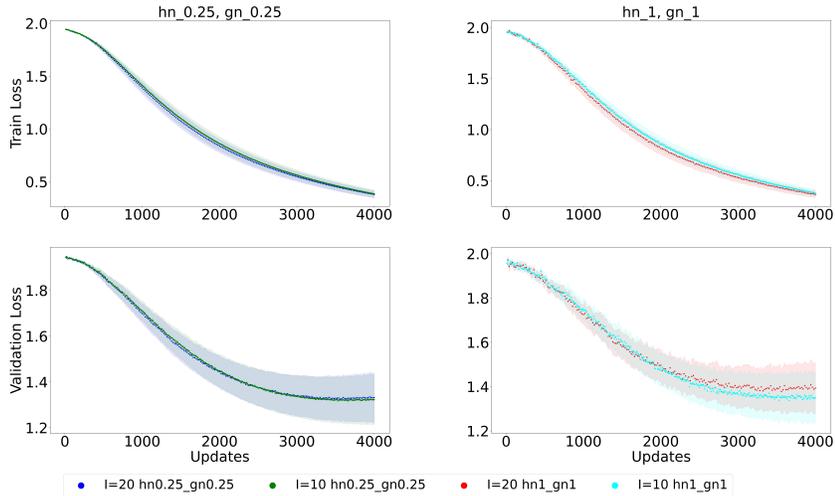}
\vspace{-12mm}
\caption{This figure shows how the choice of $I$ affects the performance when we fix the noise strategy (Noisy hidden representations and gradients) and relatively large standard deviation. In this case, Lager $I$ will increase validation loss. However, if we apply noises with a relatively small standard deviation ($0.25$), the effect caused by different $I$ will be more negligible.}
\label{fig: fixstd}
\end{figure*}




\clearpage

\begin{figure*}[t!]
\centering
\includegraphics[scale = 0.154]{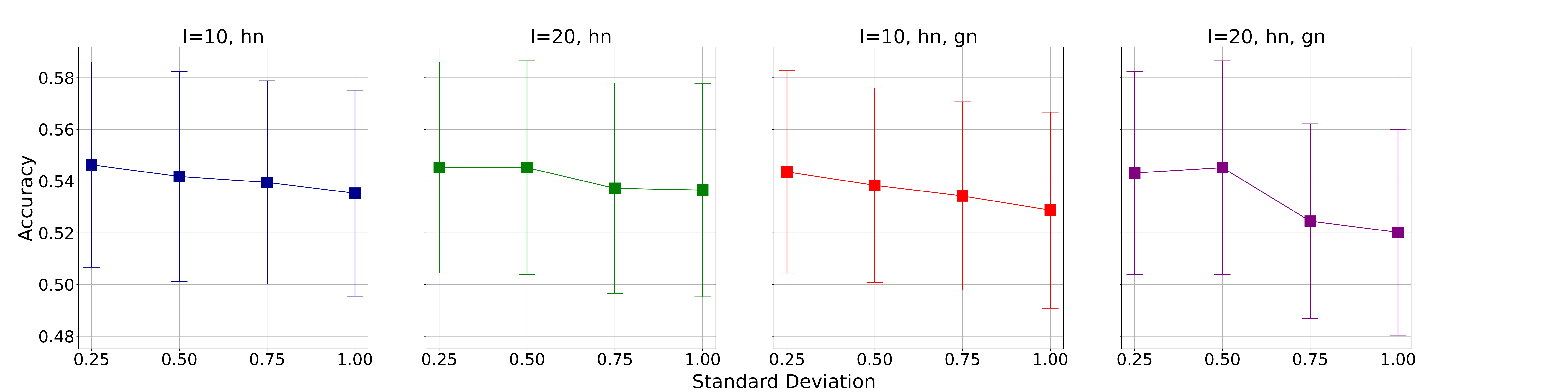}
\caption{This plot gives the average test accuracy and the corresponding 95\% confidence interval for each combination between $I \in \{10,20\}$, noise strategies, and standard deviations. We see that, in general, when fixing $I$ and the noise strategy, a larger standard deviation will result in lower average test accuracy.}
\label{fig: effectofstd}
\end{figure*}

\subsubsection{Summary Table for All Non-DP Experiments}

A test set summary Table \ref{table: table_all} for all Non-DP experiments is provided in this section. The experimental results using our method of $I=10$ \textbf{without gradient compensation} (i.e. only upload $h_k^{t}$ to the central server for each client $k$) are also added to the summary table for reference. For deterministic node classification, and supervised classification tasks, we use the same number of updates, learning rate, and batch size in \ref{Appendix: model_description} for our method without gradient compensation. For the stochastic node classification task on synthetic graphs, we change the learning rate from $0.2$ to $0.6$ and the number of updates from $5000$ to $8000$ for our method without gradient compensation.

\begin{table}[th!]
\caption{\footnotesize Summary table for average test accuracy and $95\%$ confidence interval on all Non-DP experiments. Hyphens indicate a specific model is not applicable under certain task, and where \texttt{V1} represents no gradient compensation. We denote deterministic node classification as DNC, stochastic node classification as SNC, and supervised classification as SC.}
\label{table: table_all}
\centering
\begin{tabular}{  c  c  c  c  c } 
\hline
& Synthetic DNC &  SubCora DNC &  Synthetic SNC &  Synthetic SC\\
\hline
\texttt{GFL-APPNP} $I=1$  & $93.2 \pm 0.92\%$ & $54.2 \pm 3.69\%$ & $98.7 \pm 0.26\%$ & $70.0 \pm 0.32\%$\\
\texttt{GFL-APPNP} $I=10$ & $93.4 \pm 0.99\%$ & $54.1 \pm 3.72\%$ & $92.4 \pm 0.19\%$ & $70.0 \pm 0.36\%$\\
\texttt{GFL-APPNP} $I=20$ & $93.3 \pm 0.94\%$ & $\boxed{54.3 \pm 3.73\%}$ & $92.5 \pm 0.17\%$ & $70.0 \pm 0.30\%$\\
\texttt{GFL-APPNP} $I=50$ & $93.0 \pm 0.96\%$ & $54.0 \pm 3.73\%$ & $92.5 \pm 0.17\%$ & $\boxed{70.2 \pm 0.33\%}$\\
\texttt{GFL-APPNP-V1} $I=10$ & $82.3 \pm 2.09\%$ & $47.3 \pm 3.81\%$ & $90.7  \pm 0.28\%$ & $69.0 \pm 0.45\%$\\
\texttt{APPNP} & $93.2 \pm 0.92\%$ & $54.2 \pm 3.69\%$ & --- & --- \\
\texttt{GCN} & $\boxed{95.2 \pm 0.54\%}$ & $51.9 \pm 3.78\%$ & --- & --- \\
\texttt{GAT} & $93.3 \pm 1.03\%$ & $47.9	\pm 3.01\%$ & --- & --- \\
\texttt{GraphSAGE} & $70.2 \pm 4.21\%$ & $47.0 \pm 3.73\%$ & --- & --- \\
\texttt{FedMLP} $I=10$  & --- & --- & --- & $61.0 \pm 0.54\%$\\
\texttt{FedMLP} $I=20$ & --- & --- & --- & $61.0 \pm 0.46\%$\\
\texttt{FedMLP} $I=50$ & --- & --- & --- & $70.0 \pm 0.32\%$\\
\texttt{MLPs}  & --- & --- & --- & $61.0 \pm 0.60\%$\\
\hline
\end{tabular}
\end{table}

\clearpage

\end{document}